\documentclass[twoside]{IEEEtran}

\usepackage[utf8]{inputenc} 
\usepackage[T1]{fontenc}    
\usepackage{hyperref}       
\usepackage{url}            
\usepackage{booktabs}       
\usepackage{amsfonts}       
\usepackage{nicefrac}       
\usepackage{microtype}      
\usepackage{times}
\usepackage{hyperref}
\usepackage{url}
\usepackage{bbm}
\usepackage{algorithm,algpseudocode}
\usepackage{booktabs}       
\usepackage{amsfonts}       
\usepackage{nicefrac}       
\usepackage{microtype}      

\usepackage{dsfont}
\usepackage{color}
\usepackage{diagbox}
\usepackage{array,multirow,graphicx}
\usepackage{mathtools}
\usepackage{amsthm,amsmath,amssymb,graphicx}%
\usepackage{enumitem}
\setlist{nosep}
\usepackage{thmtools}
\usepackage{thm-restate}
\usepackage{subfiles}	
\usepackage{subcaption}

\DeclareMathOperator*{\argmax}{arg\,max}

\DeclareMathOperator*{\support}{supp}

\newtheorem{theorem}{Theorem}

\newtheorem{assumption}{Assumption}

\DeclareMathOperator*{\E}{\mathbb{E}}
\newcommand{\Ave}[1]{\E_{#1}}

\def\Wmat{A}
\newcommand{\WeightMatrix}[1]{\Wmat^{(#1)}}
\newcommand{\W}[1]{\WeightMatrix{#1}}

\newcommand{\bias}[1]{b^{(#1)}}

\def\width{\Omega}

\def\sign{sign}
\def\defprednet{\sign\circ\net}
\def\prednet{\net^{\sign}_w}

\newcommand{\sig}[1]{\sigma^{(#1)}}
\def\sigl{\sig{l}}
\def\sigbar{\bar{\sigma}}

\def\net{\mathcal{N}}

\def\WeightSpace{\mathcal{W}}
\def\PathSpace{\mathcal{F}}

\def\bbR{\mathbb{R}}

\def\pdata{\mathcal{D}}

\newcommand{\risk}[2]{R_{#1}(#2)}

\usepackage{bbm}
\newcommand{\indicator}[1]{\mathbbm{1}_{\{#1\}}}
\def\wpath{\bar{w}}

\def\trainS{S^{(m)}}
\def\trainSsup{S^{(s)}}
\def\trainSrest{S^{(m-s)}}
\newcommand{\samplex}[1]{x^{#1}}
\newcommand{\sampley}[1]{y^{#1}}
\def\samplexj{\samplex{j}}
\def\sampleyj{\sampley{j}}
\def\trainset{\{(\samplex{j},\sampley{j})\}_{j=1}^m}
\def\trainembed{\{(\phi(\samplex{j},w),\sampley{j})\}_{j=1}^m}

\usepackage{soul,xcolor}
\setstcolor{red}

\newcommand{\norm}[1]{\lVert #1\rVert}

\def\Azte{Assumption \ref{asm:zte}}
\def\Aone{Assumption \ref{asm:mm_relax}}

\def\pathindex{{i_d,\ldots ,i_1,i_0}}
\def\neuronindex{i_d,\ldots ,i_1}

\def\indexbounds{i_l\in [b] \text{ for } l\in [d] \text{ and } i_0\in[f]}

\def\choosemany{\binom{m}{s}}

\def\InfoSpace{\mathcal{M}}
\def\HypSpace{\mathcal{H}}

\def\hyp{\xi}
\def\idx{\textbf{i}}
\def\tset{\trainS}
\def\isp{\InfoSpace}
\newcommand{\Prob}[1]{\operatorname*{\mathbb{P}r}\limits_{\tset\sim\pdata^m}\left[#1\right]}
\def\Risk{\risk{\pdata}{\prednet}}

\def\skisp{\pi}
\def\pth{j}
\def\pthhat{j_{\net}}
\def\Is{I^{(s)}}
\def\Pi{P^{I}}
\def\Ps{P^{\sigma}}
\def\Ppi{P^{\skisp}}
\def\Qi{Q^{I}}
\def\Qs{Q^{\sigma}}
\def\Qpi{Q^{\skisp}}
\def\Ron{\mathcal{R}}
\def\Uniform{Uniform}
\def\cond{C_{MM}}
\def\pthh{\pth^{\textnormal{th}}}

\begin{document}

\title{Sample Compression, Support Vectors, and Generalization in Deep Learning}

\author{Christopher~Snyder,~\IEEEmembership{Student Member,~IEEE,}
        and~Sriram~Vishwanath,~\IEEEmembership{Senior~Member,~IEEE}
\thanks{Chris Snyder and Sriram Vishwanath are with the Department
of Electrical and Computer Engineering, University of Texas, Austin USA e-mail: christopher.g.snyder@utexas.edu, sriram@utexas.edu}}


\maketitle

\begin{abstract}
Even though Deep Neural Networks (DNNs) are widely celebrated for their practical performance, they possess many intriguing properties related to depth that are difficult to explain both theoretically and intuitively. Understanding how weights in deep networks coordinate together across layers to form useful learners has proven challenging, in part because the repeated composition of nonlinearities has proved intractable. This paper presents a reparameterization of DNNs as a linear function of a  feature map that is {\em locally independent} of the weights. This feature map transforms depth-dependencies into simple {\em tensor} products and maps each input to a discrete subset of the feature space. Then, using a max-margin assumption, the paper develops a {\em sample compression} representation of the neural network in terms of the discrete activation state of neurons induced by $s$ ``support vectors''. The paper shows that the number of support vectors $s$ relates with learning guarantees for neural networks through sample compression bounds, yielding a sample complexity of $\mathcal{O}(ns/\epsilon)$ for networks with $n$ neurons. Finally, the number of support vectors  $s$ is found to monotonically increase  with width and label noise but decrease with depth.
\end{abstract}


\section{INTRODUCTION}
Neural networks represent an intriguing class of models that have achieved state-of-the-art performance results for many machine learning tasks. Although neural networks have been studied for over half a century \cite{McCulloch1943}, the variations which have recently garnered interest are called ``deep" neural networks (DNNs). Deep learning is characterized by stacking one layer after another and using the  computational power of modern graphical processor units (GPUs) or custom processors/ASICs to train them.  It is shown experimentally that such networks with more layers tend to generalize better 
\cite{novak2018} 
\cite{Neyshabur2018}. 

Thus, deep learning presents a scenario where the best performing models are also generally poorly understood. Improvements to our theoretical understanding of deep neural networks will aid in structured, principled approaches to the design, analysis, and use of such networks. An important step in understanding a model is to prove generalization bounds that agree with performance in practice. For DNNs, several attempts have been made in this direction based on margin, perturbation, PAC-Bayes, or complexity type approaches (see related work section for further details). 

Proving generalization bounds for our existing models helps us to build better models in the future by forcing us to articulate and analyze what exactly it is about DNNs that makes them work. From classical learning theory, our intuitions is to associate regularity with small parameter number or norm penalization. We expect that discussion of the interplay between depth and generalization of unregularized networks will be particularly fruitful exactly because it runs contrary to these intuitions. We seek to add an alternative perspective by compressing the DNN classifier as one of the solution to one of a small number of optimization problems, side-stepping the need to discuss the norm of various weight matrices or the architecture dimensions explicitly.

In this paper, under assumptions presented in Section \ref{sec:assumptions}, we recast Leaky-ReLU type networks as an equivalent support vector machine (SVM) problem where the features correspond to paths through the network and the embedding map $\phi$ has local invariance to perturbation of the weights of the DNN. Though this embedding is a non-trivial function of these weights, the induced kernel has a simple interpretation as the inner product in the input space scaled by the number of shared paths in the network. 

Our main contributions can be summarized as:

\begin{enumerate}
\item{We present a framework for recasting neural networks with two-piecewise-linear nonlinearities (such as ReLU) as an SVM problem where classification with the network is equivalent to linear classification in a particular tensor space. Here, the corresponding embedding of training points is insensitive to local perturbations of weights.}
\item{We introduce for study a particular type of DNN satisfying a max-margin assumption. These networks can be compressed as the solution to an optimization problem determined by a few samples (much like SVMs). We point through analogy with unregularized logistic regression (in the feature space) that we may expect DNNs trained with unregularized gradient to satisfy this assumption in the training limit. We show empirically that modifying trained DNNs to satisfy this max-margin assumption essentially leaves their predictions unchanged.}

\item{Under a max-margin assumption on the network within this new feature space, we define  ``network support vectors'': those training samples that are mapped to support vectors under the learned embedding map. An important consequence of this max-margin assumption is that only \textit{finitely many} neural network classifiers correspond to a set of network support vectors.}
\item{We show that the number of network support vectors, $s$, can be related theoretically to generalization and experimentally to network architecture. We use a sample compression variant of PAC-Bayes to prove in Theorem \ref{thm:pbsc} a bound on the sample complexity of max-margin networks with $n$ neurons of $\mathcal{O}(ns)$. We show that the quantity, $ns$, experimentally decreases with depth (despite $n$ increasing with depth). We expect the quantity $n$ to be subject to further improvement.}
\end{enumerate}

\section{RELATED WORK}
There have been multiple, well-thought-out efforts to model, characterize and understand generalization error in DNNs. One well-studied direction is to impose a sufficiently small norm condition on the neural network weights
 \cite{Golowich2017}
 \cite{Neyshabur2015}. 
Since the weight norms induce a bound on the network's Lipschitz constant, one can connect this with  insensitivity of the network output to input perturbations, either through a product of weight spectral norms 
\cite{Bartlett2017} 
or through the norm of the network Jacobian itself 
\cite{Sokolic2017}. 
Instead of invariance to input perturbations, one can also consider the degree of invariance of the network dependence to weight perturbations \cite{Neyshabur2017}. 
A natural way to concretely relate such perturbation schemes to generalization error is through the means of probably approximately correct PAC-Bayes analysis as in
\cite{McAllester1999} 
\cite{McAllester2013}. 

The general principle underlying PAC-Bayes analysis is to characterize (in bits, with respect to some prior) the degree of precision in specifying the final neural network weights in order to realize the observed training performance. Such PAC-Bayes based generalization bounds were applied successfully to the study of neural networks by \cite{Langford2002} and more recently by \cite{Dziugaite2017}, albeit for stochastic networks. Generally speaking, if multiple weights corresponding to a large neighborhood result in similar neural network behavior, then fewer bits are needed to specify these weights. Overall, insensitivity to weight perturbation is one potential manner to formalize the popular high-level idea that "flat minima generalize well" \cite{Neyshaburc,Hochreiter1997}.

In order to correctly reproduce the improvement in generalization observed in deep learning with each additional layer, the principle difficulty is that these approaches must make layer-wise considerations (either of each weight matrix or each layer-wise computation) that accumulate and grow the generalization bound as depth increases. Of course, it is possible to find suitable assumptions that control or mitigate this depth-dependent growth as in 
\cite{Golowich2017} 
or \cite{Kawaguchi2017}. 
Given this challenge,  other "network compression" type approaches that characterize the network function without addressing every individual parameter are gaining interest. For example, 
\cite{Neyshabur2014} 
analyzes the number of nonzero weights as a form of capacity control, while others have studied approximating a deep network by a "compressed" version with fewer nonzero weights
\cite{Arora2018} 
\cite{Zhou2018}. 

In this paper, we use a \textit{sample} compression 
\cite{Littlestone1986} 
representation approach for understanding neural network depth-dependence. We transform the neural network into a related SVM problem, then recover the network function from (suitably defined) support vectors. Sample compression characterizes PAC learnable functions as those that can be recovered from a small enough subset of training samples 
\cite{Floyd1995}. 
This theory readily finds application in kernel modeling, for example 
\cite{Germain2011}, 
since support vectors provide a natural correspondence between (max-margin) hypotheses data subsets. Our own work makes consistent use of machinery developed by Laviolette 
\cite{Laviolette2007} 
who significantly increased the applicability of sample-compression theory by allowing model reconstructions to depend additionally on "side channel" information.

\section{ON NEURAL NETWORKS AS SUPPORT VECTOR MACHINES}
\subsection{Notation Definitions and Setting}

In this paper, we consider the family of nonlinearities  $\rho(x)=\beta x\indicator{x<0}(x)+\gamma x\indicator{x\geq0}(x)$ for $\beta,\gamma\in\bbR$ for the neural network, which encompasses ReLU, LeakyReLU, and absolute value as examples. We will refer to these nonlinearities collectively as "Leaky-ReLU". For vector arguments, $\rho$ is understood to be applied element-wise. We do not use biases. For integer $m$, we will use $[m]$ to mean the set $\{1,\ldots,m\}$.

Consider a neural network with $d$ (\underline{d}epth) hidden layers, width $\width$  neurons in \underline{l}ayer $l$, $f$ input \underline{f}eatures, and $m$ training sa\underline{m}ples. 

We will use $\WeightSpace=\bbR^{\width}\times(\prod_{i=1}^{d-1}\bbR^{\width\times\width})\times\bbR^{\width\times f}$ to denote the set of all possible weights within the neural network. Here $\WeightMatrix{l}_{i_{l+1},i_l}$ refers to the scalar weight from neuron $i_l$ in $l$ to neuron $i_{l+1}$ in layer $l+1$. We use $w$ to refer to \textit{all of the weights collectively}, with $w=(\WeightMatrix{d},\ldots,\WeightMatrix{1},\WeightMatrix{0})\in\WeightSpace$. Each $w$ corresponds to a neural network mapping from each $x\in\mathcal{X}\triangleq\bbR^f$ to $\bbR$ as follows:

\begin{equation}
\net(x,w)\triangleq\net_w(x)\triangleq\WeightMatrix{d}\rho(\WeightMatrix{d-1}\rho(\ldots(\WeightMatrix{1}\rho(\WeightMatrix{0}x)\ldots))
\end{equation}

We distinguish between $\net_w:\mathcal{X}\mapsto\bbR$, which returns scalar values, and the related classifier returning labels, $\prednet\triangleq\defprednet_w:\mathcal{X}\mapsto\mathcal{Y}$. Here $\mathcal{Y}\triangleq\{-1,+1\}$ and $\sign(\cdot)$ is a function returning the sign of its argument (defaulting to $+1$ for $0$ input).
For a data distribution $\pdata$ on $\mathcal{X}\times\mathcal{Y}$, the goal is to use a training set $\trainS=\trainset\sim\mathcal{D}^m$ to learn a set of weights $w$ so that $\prednet$ has small probability of misclassification on additional samples drawn from $\mathcal{D}$.

We define a {\em path} (in a neural network) to be an element of 
$(\prod_{i=1}^{d}[\width])\times[f]$, corresponding to a choice of $1$ neuron per hidden layer and $1$ input feature. Sometimes it is convenient to refer to these input features as neurons in layer $l=0$. Thus, one says that the path $\pathindex$ traverses neuron $i_l$ in layer $l=0,1,\ldots,d$.

Given a set of weights $w$, we define $\Lambda(w)$ to be the path-indexed vector with the product of weights along path $p$ in position $p$. Often we use $\wpath$ to shorten $\Lambda(w)$, and we use $\wpath_p$ or $\wpath_{\pathindex}$ when we want to specify the path.

\subsection{A Reparameterization of the Network}\label{sec:reparameterization}

Consider the set of all paths starting from some feature in the input and passing through one neuron per hidden layer of a ReLU neural network. Index these $f\width^d$ many paths by the coordinate tuple $(i_d,\ldots,i_1,i_0)$ to denote the path starting at feature $i_0$ in the input and passing through neuron $i_l$ in hidden layer $l$. Given a set of network weights $w$, we can define $\Lambda(w)=\wpath=\wpath_{\pathindex}$, whose $(\pathindex)^{th}$ coordinate is the product of weights along path $(\pathindex)$. Inspired by \cite{Kawaguchi2017}, (who used a similar factorization without exploring the connections with support vector machines) we note that the output of a neural network can be viewed as a sum of contributions over paths
\[
\net(x,w)=\sum_{p=(\pathindex)} \sig{d}(x,w)_{i_d}\cdots\sig{1}(x,w)_{i_1} x_{i_0}\wpath_{p}
\]
where $\sig{l}(x,w)$ is an indicator vector for which neurons in layer $l$ are active for input $x$ with weights $w$. For convenience, we also define $\sigbar(x,w)=\sigbar(x,w)_{\neuronindex}=\sig{d}(x,w)_{i_d}\cdots\sig{1}(x,w)_{i_1}$, which is also an indicator but over paths instead of neurons\footnote{Leaky-ReLU units scale inputs by $\beta$ or $\gamma$ in place of $0$ or $1$, hence the $\sigl$ are no longer literally indicator functions}. 
The above summation over all tuples $(\pathindex)$ can be interpreted as an inner product $\langle \phi(x,w),\wpath\rangle$ where 
\begin{equation}\label{eqn:phi}
\phi(x,w)_{\pathindex}=\sig{d}(x,w)_{i_d}\cdots\sig{1}(x,w)_{i_1}x_{i_0} 
\end{equation}
is a $w$-parameterized family of embedding maps from the input to a feature space we denote as the "Path Space" $\PathSpace$, i.e., the set of all tensors assigning some scalar to each path index-tuple $\pathindex$ with $\indexbounds$. The neural network then is \textit{almost} a kernel classifier in that the model only interacts with the input through inner products with a feature map $\phi(x,w)$. Though unlike a SVM, the feature map has some dependence on $w$.

\def\CondStop{C_{STOP}}
\def\Ballw{B_\epsilon(w)}
\def\netrestrict{\net\restriction_{\trainS\times B_\epsilon(w)}}
\def\netparmrest{\net\restriction_{B_\epsilon(w)}}
\def\wfinal{w_f}
An important insight is that, over small regions of the weight space, our embedding $\phi(x_i,w)$ does not depend on $w$ for any of the finitely many training points. More precisely, suppose that none of the pre-nonlinearity activations of neurons in $\net$ are \textit{identically} zero. Then for each training sample and each neuron pre-activation, we obtain an open ball about this pre-activation (excluding zero). Since the function from the weights to each pre-activation is continuous, the preimage of each ball in the weight space is open. The intersection of these (finitely many) preimages is an open set around the current network weights in which the feature space embedding of training samples (not necessarily test samples) is \textit{independent} of our weights. Interestingly, this implies that over small, say $\epsilon>0$ sized, regions of weights around $w$, say $\Ballw$, we may parameterize our training outputs unambiguously by the product of weights over paths, $\bar{w}\in \Lambda(\Ballw)$, instead of the "usual" parameterization $w$. Note though that globally the relationship is not $1-1$. 

Though the practical size of these regions may be very small, this local linearization in terms of $\wpath$ is analytically well suited for characterizing networks trained by local methods, e.g. gradient descent. Indeed, the gradient operator at every $w$ treats $\sigbar$ as a constant function of $w$, independently of the size of the containing region. If ultimately, the final training weights satisfy some condition in terms of the gradient operator then we should study this condition by considering local perturbations of a set of linear models parameterized by $\wpath$. We assume cross-entropy loss, in which case the local loss landscape of models in a small neighborhood of $\wpath$ is exactly that of the loss landscape of logistic regression models on $\PathSpace$ with the same training data and feature map $\phi(\cdot,w)$.

\subsection{Assumptions Made}
\label{sec:assumptions}
Prior to detailing the assumptions made in this paper, we first highlight a compelling recent work on \textit{unregularized} logistic regression for linearly separable problems in \cite{Soudry2017a}. Here, the authors prove that gradient descent yields a sequence of classifiers whose normalized versions converge to the max margin solution. For example, the authors provide a theoretical basis for the increase in test accuracy and test loss during training even after the training accuracy is 100\%. Note that this peculiar behavior is also common to neural networks \cite{Shwartz-Ziv2017}. Inspired by this connection, we assume the following:

\begin{assumption}{Zero Training Error}\label{asm:zte}\\
The weights $w$ obtained from training on $\trainS$ ensure $\prednet$ correctly classifies every sample in $\trainS$. Equivalently:
\[\forall(x,y)\in\trainS\quad y\langle\Lambda(w),x\rangle\geq0\]
\end{assumption}

Note that, for zero training error, linear separability\footnote{In fact we are guaranteed a separating hyperplane containing the origin} of our embedded data, $\trainembed$, is strictly necessary. Motivated by analogy with maximum margin classifiers in logistic regression, we make the following second assumption on the network weights obtained by training on $\trainS$:

\def\trainembed{(\phi(\samplex{j},w),\sampley{j})_{j=1}^m}
\begin{assumption}{Max-Margin}\label{asm:mm_relax}\\
The training procedure returns weights $w$ such that up to positive scaling, $\Lambda(w)$ is the maximum margin classifier for the $w$-parameterized embedding $\{(\phi(\samplex{j},w),\sampley{j}):j\in[m]\}$. Equivalently, $w$ must satisfy the relation
\[\Lambda(w) \in \argmax_{\bar{v}\in\PathSpace} \min_{(x,y)\in\trainS}{\frac{y\langle \bar{v},\phi(x,w)\rangle}{\norm{\bar{v}}}} \]
\end{assumption}

\subsection{Merit of Assumptions}\label{sec:rationale}
Of the two assumptions made in the paper, note that Assumption \ref{asm:zte}, of zero training error, is not uncommon for neural networks in practice \cite{Soudry2016}.   Therefore, we do not discuss Assumption \ref{asm:zte} in greater detail in this subsection, focusing more on the second assumption in this paper.

The value of Assumption \ref{asm:mm_relax} is more nuanced, and we devote an entire section to discussing this, expanding with relevant experiments in Appendix \ref{app:Aone}. In short, we show empirically that networks trained with gradient descent satisfying Assumption \ref{asm:zte} are not too different from those satisfying \Aone{}. This is not unexpected given the comparison with unregularized logistic regression (in the feature space). The merit of an assumption lies not in whether it is strictly true but in whether it is \textit{interesting}. This assumption concisely explains certain experimental phenomena and allows theoretical tractability. Experimentally, it also seems \textit{relevant} to practical DNNs. Unregularized logistic regression only finds the max-margin classifier in the training limit as the number of iterations approaches infinity. Our experimental observations of DNNs trained with finitely many iterations are consistent with those that \textit{approximately} satisfy the max-margin assumption.

Note that, without idealized assumptions such as Assumption \ref{asm:mm_relax}, it is very difficult to build a framework that helps us gain an understanding of the problem, or the implications of its solution. In particular, Assumption \ref{asm:mm_relax} forms a starting point for a deeper theoretical understanding of neural networks, one that provides useful insights that can be employed towards a more general, overall theory for deep neural networks.

\subsection{Network Support Vectors}

In this section we use \Aone~ to extend the definition of support vectors to neural networks with zero training error. 
By the Representer Theorem
\cite{Scholkopf2001}, 
the max-margin condition on $\wpath$ in \Aone~ implies that for some nonnegative scalars $\alpha_1,\ldots,\alpha_m$,
\begin{equation}\label{eqn:representer}
\wpath=\sum_{k=1}^m\alpha_k\sampley{k}\phi(\samplex{k},w).
\end{equation}

\def\nsv{{NSV}}
\def\nsvs{{NSVs}}
\def\svs{{SVs}}
\def\samplesup{(\samplex{1},\sampley{1}),\ldots,(\samplex{s},\sampley{s})}

Analogously to classical SVMs, for a fixed set of weights $w$ achieving \Aone, we define the subset $\trainSsup\triangleq\{(\samplex{k},\sampley{k}):\alpha_k\not=0\}$ of those training data points that correspond to nonzero $\alpha_k$ to be ``network support vectors"(\nsvs) or simply ``support vectors" when context is clear. We also use $\trainSrest=\trainS-\trainSsup$ to denote the $m-s$ data which are not support vectors. 

To gain an experimental understanding of these ``support vectors", we  train neural networks on a 2-class MNIST variant formed by grouping labels $0-4$ and $5-9$. We show that many qualitative properties of SVMs continue to hold true in this case when the embedding map is learned. We first determine network weights obtained from minimizing the neural network loss. Then, we define an embedding map $\phi(\cdot,w)$ using those weights. Finally, we train a SVM using the kernel as defined by $\langle \phi(\samplex{i},w),\phi(\samplex{j},w)\rangle$.
The details of this experiment and all others in this section are presented in the Appendix \ref{app:experiment}.

As noted in these experiments, we determine that the behavior of the number of NSVs  is qualitatively similar to what we might find in a conventional SVM setting. For example, we typically find $s/m\approx 0.15$. We find that every time we increase the number of training samples, $m$, and retrain the network from scratch, the net effect is that $s$ increases but $s/m$ asymptotically decreases to $0.1$ (Figure \ref{fig:sv_vs_m}). This is entirely expected in the simplified setting with a fixed embedding map: additional samples can only decrease the margin, reducing the fraction of volume within the margin of the hyperplane. Thus, additional randomly selected samples are increasingly unlikely to be support vectors. 

Given that the SVM model (with fixed embedding) is determined entirely by $\trainSsup$, the model is said to have ``memorized'' the sample $(x,y)\in\trainS$ iff $(x,y)$ is a support vector. We find that a similar notion holds for network support vectors. In deep learning, the notion of memorizing a given individual sample is less clear, but we often describe a DNN with wildly divergent test and train accuracies as having ``memorized the dataset". For example, DNNs will often achieve zero training error even when there is no relationship between inputs $\mathcal{X}$ and outputs $\mathcal{Y}$.

If we randomize each label of $\trainS$ prior to training so that the training data is  sampled from a product of marginal distributions instead, $\trainS\sim\pdata_\mathcal{X}\times\pdata_\mathcal{Y}$, 
we observe experimentally that $s/m\approx 0.6$ (Figure \ref{fig:noisy_label_sv_vs_m}). This can be understood as follows: although the labels are independent of the inputs, there are natural clusters in the input that the model can use to fit these random labels in the training data. Each sample has a label consistent with at least half of the training set, since half of the training data have the correct label. Thus, the DNN is learning a pattern corresponding to the true labeling (or its reverse) and {\em building in exceptions} for the rest of the data by adding them as support vectors. Note that learning this labeling on MNIST requires $~0.1m$ support vectors (from before). The addition of $0.5m$ training samples that violate the first learned labeling results in the observed $0.6m$ total.

In a conventional SVM setting, models with fewer support vectors are thought of as more parsimonious. Furthermore, the fraction of training samples that are support vectors can be concretely linked to generalization bounds through sample compression techniques, as in \cite{Littlestone1986}. An important observation is that the SVM solution can be reconstructed from the subset of support vectors $\trainSsup\subset\trainS$, so bounding $s=|\trainSsup|$ controls the number of training samples the model can memorize. Similarly, in Section \ref{sec:pbsc}, we construct analogous bounds for deep neural networks that depend centrally on this number of NSVs. 

We now turn to understanding how this number of NSVs varies with architecture parameters. We first study fully-connected networks on flattened MNIST images. 
There, we find that the fraction $s/m$ increases logarithmically as we increase the width $\width$ (Figure \ref{fig:sv_vs_width}) but \textit{decreases} linearly as we increase the depth $d$ (Figure \ref{fig:sv_vs_depth}). It is interesting to note that $s$ decreases with depth $d$ in these cases. Taking on faith for the moment that the next section (specifically, Theorem \ref{thm:pbsc}) will fashion a bound on the test error of the form $\mathcal{O}(ns/m)$, (as initially advertised in the abstract), we can understand the significance of this observation. Although $n$ increases linearly with $d$, we observe a net decrease in the generalization bound with depth, since the decrease in $s$ with depth is "superlinear" in the sense that doubling the number of layers from $3$ to $6$ more than halves $s$. As a result, the closely related generalization bound we will justify in the subsequent section seems to decrease with depth as well (Figure \ref{fig:thm1_vs_depth}). While these experimental relationships are interesting, these are preliminary in nature, and additional study is required to make concrete claims on the relationships between parameters of the network.

Of particular interest is the inverse relationship between the number of NSVs and depth. In order to understand whether this relationship continues to hold in more general settings, we study binary classification of Frogs vs Ships on the CIFAR-10 dataset using convolutional networks with nonzero biases. These networks consist of initial convolutional and max pooling layers followed by a variable number of FC depth many fully-connected layers. We see that the relationship is more noisy, but still there is a clear trend that $s$ decreases significantly for larger depths (Figure \ref{fig:cifar_sv_vs_depth}). While we have extended the notion of network support vectors to convolution and nonzero bias networks (Appendix \ref{app:bias_conv}), our generalization theory developed in the next section only supports fully-connected networks for now. Therefore we don't calculate a bound such as in Figure \ref{fig:thm1_vs_depth} for this data.

\begin{figure}[h]
\centering
\includegraphics[width=\linewidth]{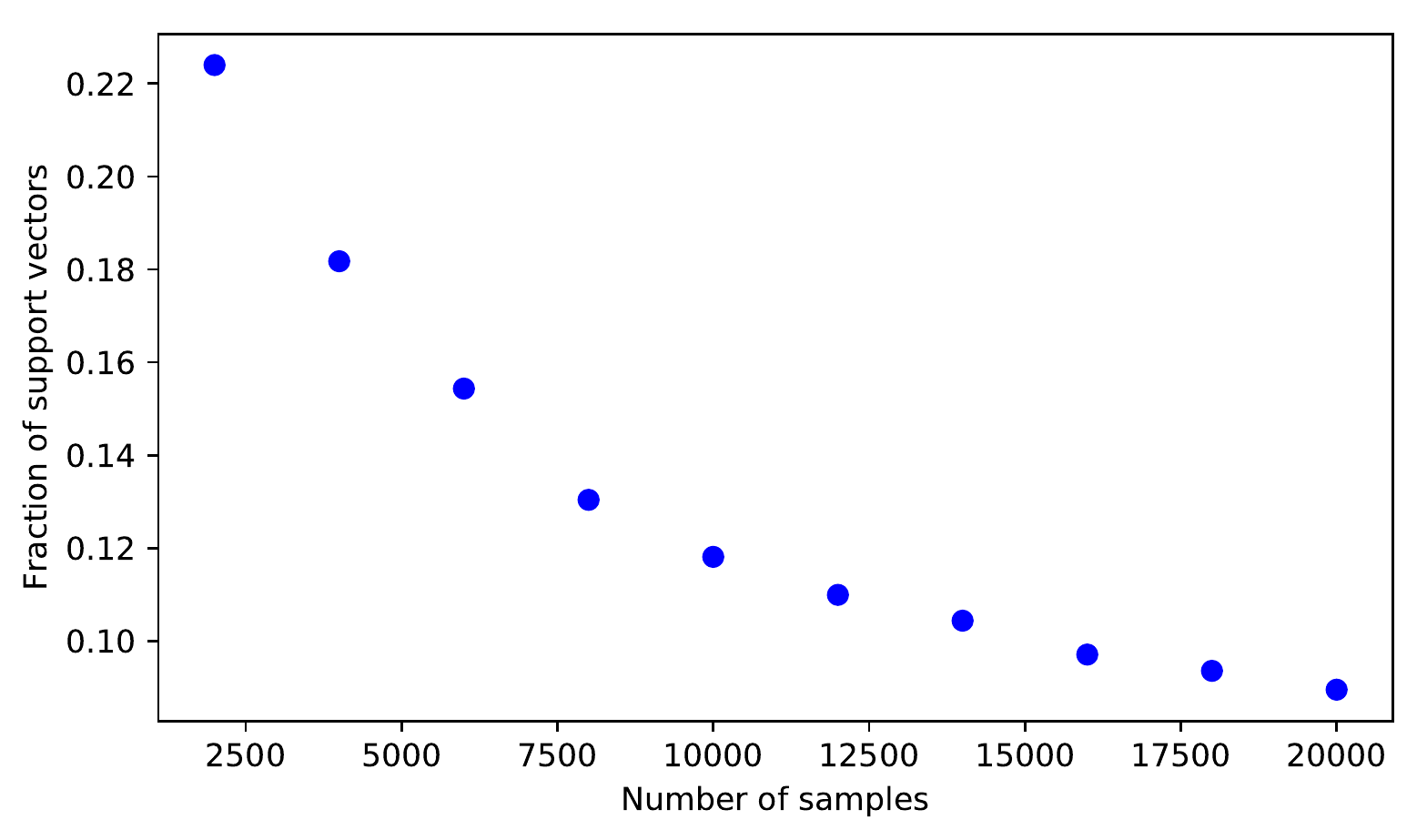}
\caption{Support vector fraction of data $s/m$ vs Number of samples m: Increasing the size of the training set decreases asymptotically the fraction $s/m$ of support vectors.}
\label{fig:sv_vs_m}
\end{figure}

\begin{figure}[h!]
\centering
\includegraphics[width=\linewidth]{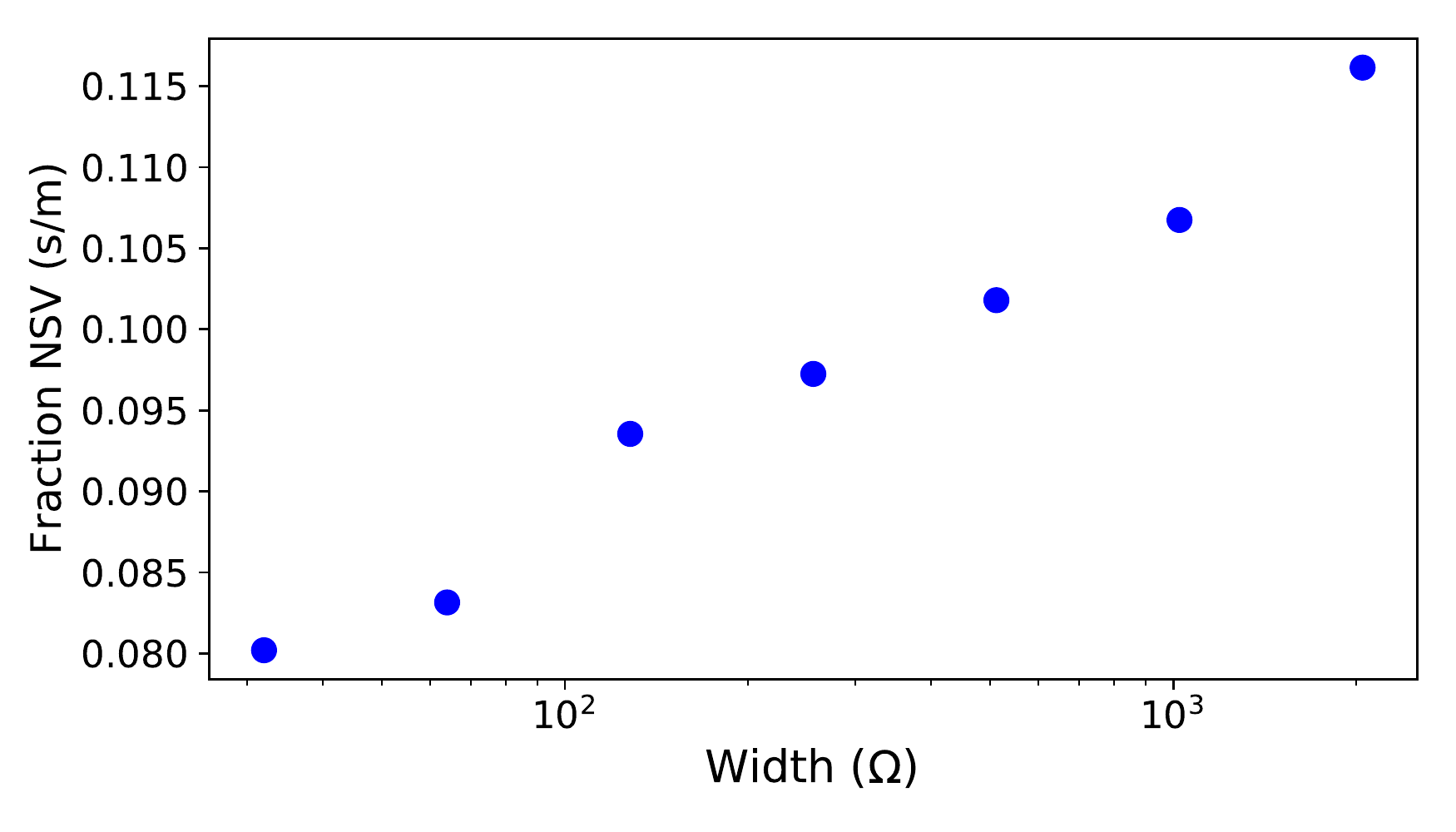}
\caption{Fraction Network Support Vectors $(s/m)$ vs width $\width$: ReLU networks of varying width $\width$ are trained to classify MNIST images. Each width-dependent trained set of network weights, $w$, is used to define an embedding $\phi(\cdot,w)$. The number of support vectors, $s$, corresponding to the maximum margin classification of $\trainembed$ is measured ($m$ is constant). Each point represents an average of three runs. The results indicate that $s$ grows proportionally to $\log(\Omega)$.}
\label{fig:sv_vs_width}
\end{figure}

\begin{figure}
\centering
\includegraphics[width=\linewidth]{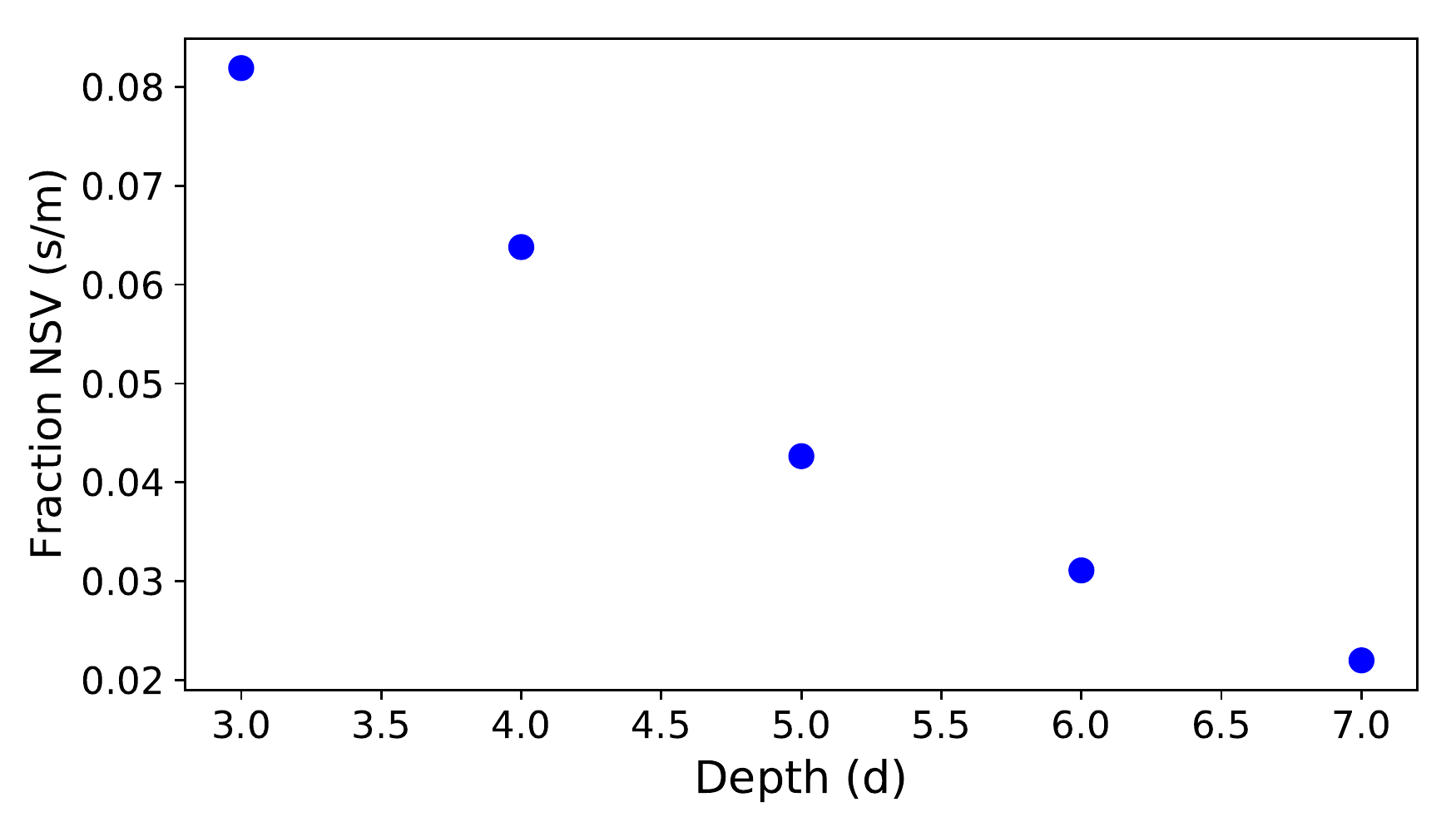}
\caption{Fraction Network Support Vectors $(s/m)$ vs depth $d$: The depth, i.e., the number of hidden layers, is varied, resulting in a depth-dependent embedding of the training data, $\trainembed$, where $w$ is the set of weights obtained from training a DNN with $d$ layers to classify data in $\trainS$. The number of support vectors $s$ decreases with depth over these finite ranges. Each point represents an average of three runs.}
\label{fig:sv_vs_depth}
\end{figure}

\begin{figure}
	\centering
	\includegraphics[width=\linewidth]{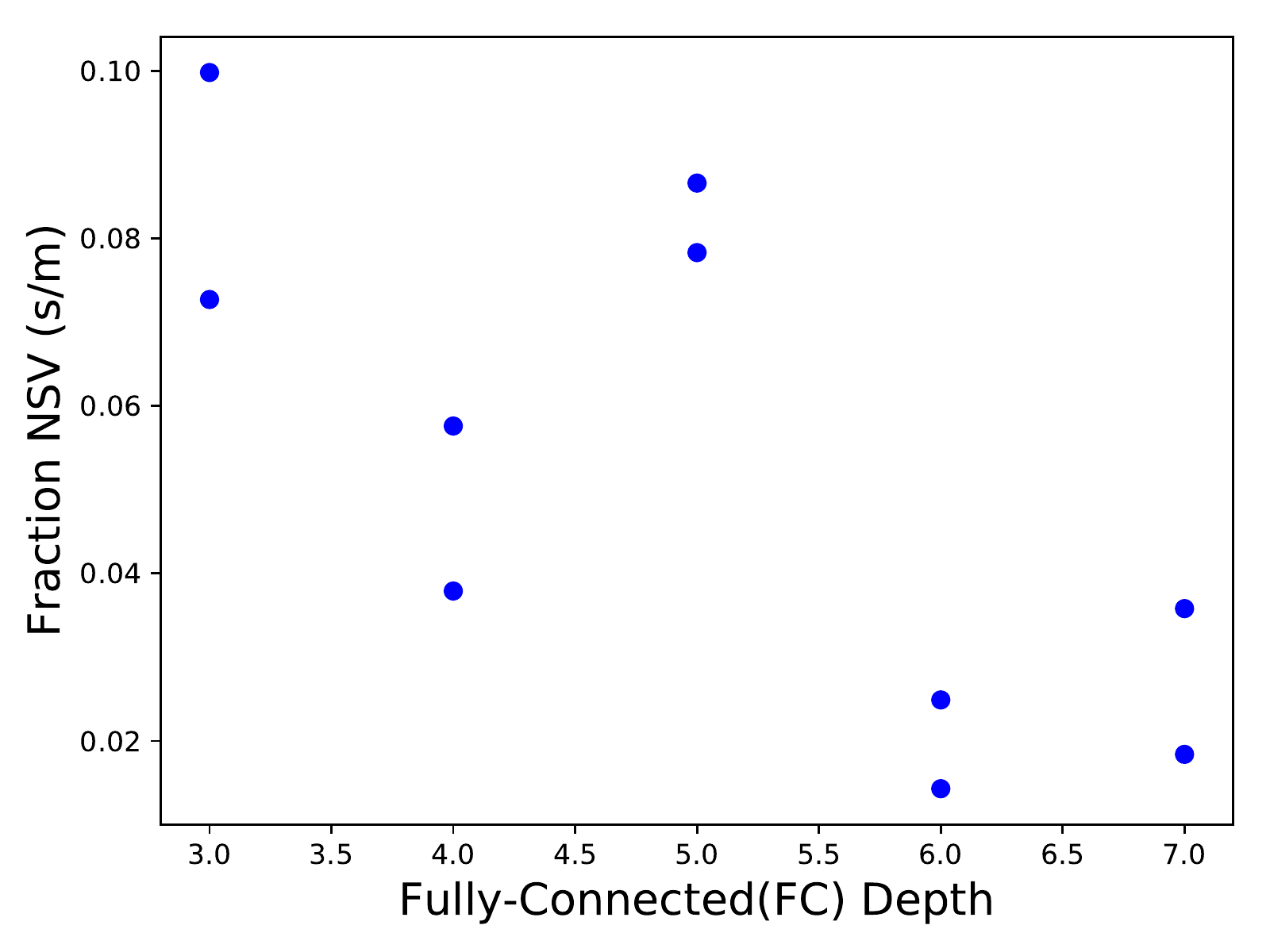}
	\caption{Fraction Network Support Vectors $(s/m)$ vs fully-connected (FC) depth on the CIFAR dataset. The first three layers learned are convolutional and are not counted toward the depth. We see that the max-margin classification of $\trainembed$ results in many fewer NSVs when then depth is made larger.
	}
	\label{fig:cifar_sv_vs_depth}
\end{figure}

\section{SAMPLE COMPRESSION BOUNDS}\label{sec:pbsc}
\def\Obound{\mathcal{O}(sn)}
\newcommand{\Ln}[1]{\ln\left(#1\right)}

In this section, we present a concrete theoretical relationship between the number of network support vectors and a bound of the test error of deep neural networks satisfying assumptions as outlined in Section \ref{sec:assumptions}. The setting for all theorems will be Leaky-ReLU (incl. ReLU) networks with arbitrary, fixed fully-connected architecture. 
Just as a SVM max-margin classifier is determined entirely by its cast of support vectors, \textit{only finitely many} neural networks satisfying the max-margin assumption (\Aone) correspond to a given set of at most $s$ network support vectors. 
This is presented as the following theorem (proof given in Appendix \ref{app:pbsc}):

\def\Bound{\mathcal{F}}
\def\BoundArgs{\Bound(m,n,s,\delta)}
\def\BoundDef{\frac{n+ns+s+s\Ln{\frac{m}{s}} + \Ln{\frac{1}{\delta}} }{m-s}}
\begin{restatable}{theorem}{PBSC}\label{thm:pbsc}
Let $\net$ refer to a Leaky-ReLU neural network with $d$ hidden layers each consisting of width $\width$ neurons so that we have $n=d\width$ neurons total. Let the weights $w$ be  deterministic functions of $\trainS$, which is a set of $m$ i.i.d. data samples from $\pdata$. Let $s<m$ be a fixed integer which does not depend on $\trainS$. 
Supposing that:
\begin{enumerate}
\item{\Azte~(Zero training error): $\prednet(x)=y$ $\forall(x,y)\in\trainS$,}
\item{\Aone~(Max-margin): $\Lambda(w)$ is some positively scaled version of the max-margin classifier for
$\{(\phi(x,w),y):(x,y)\in\trainS\}$, and}
\item{(At most $s$ support vectors): $\Lambda(w)=\sum_{k=1}^m\alpha_k\sampley{k}\phi(\samplex{k},w)$ for some set of coefficients $\alpha_k$, at most $s$ of which are nonzero.}
\end{enumerate}
 then we have, $\forall\delta\in(0,1]$

\begin{align}
\Prob{\Risk \leq  \Bound(m,n,s,\delta)}\geq 1-\delta\nonumber
\end{align}
where
\begin{align}
\Bound(m,n,s,\delta)&=\BoundDef\label{eqn:bound} \\
&\approx \frac{ns+\Ln{\frac{1}{\delta}}}{m}\nonumber
\end{align}
\end{restatable}

\begin{figure}[h!]
\centering
\includegraphics[width=\linewidth]{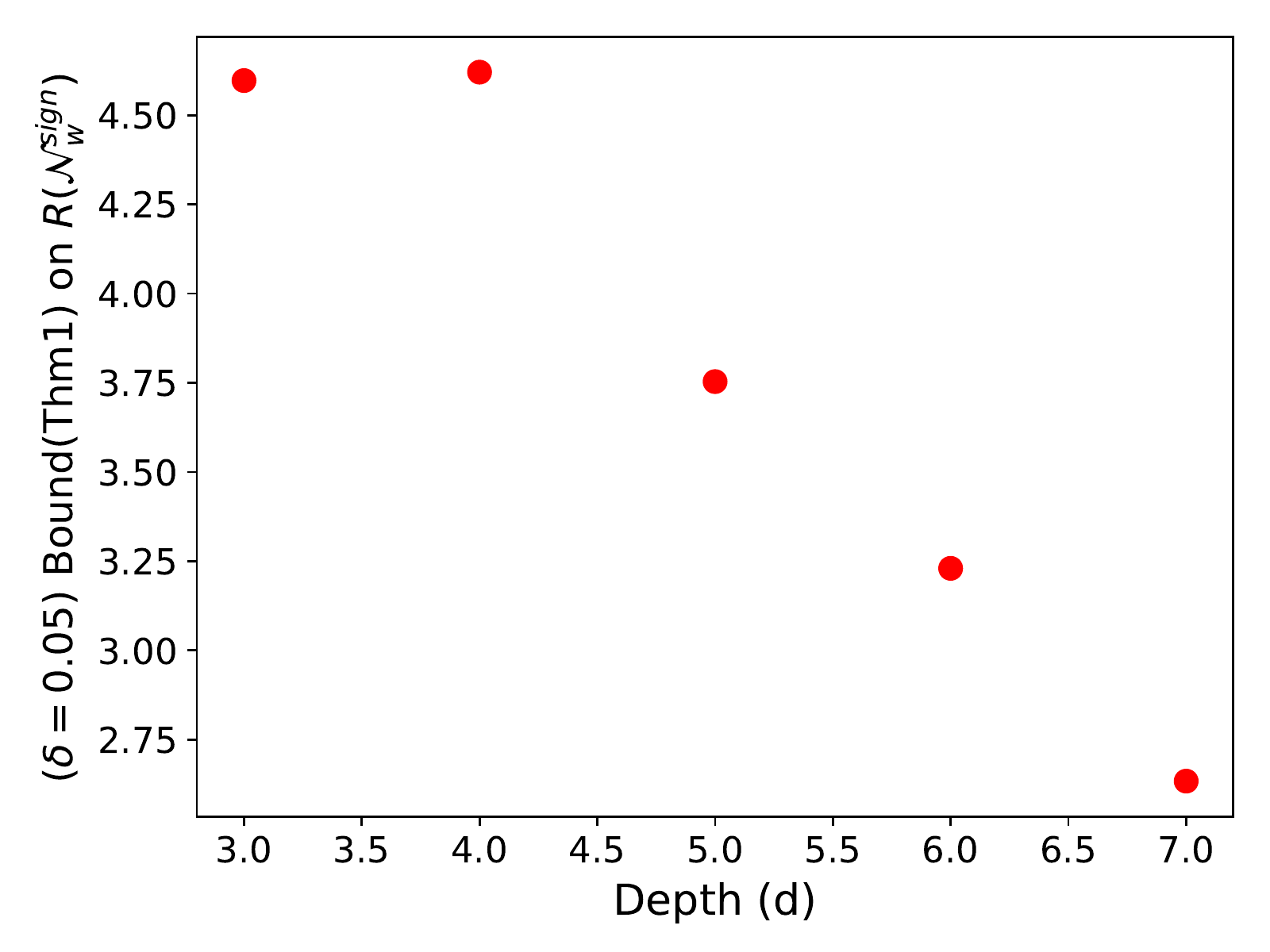}
\caption{Numerical Value of the Risk Upper Bound, $\BoundArgs$, in Theorem \ref{thm:pbsc} as Depth $d$ Varied: The outcomes and values of the  experiments in the previous section (see Appendix \ref{app:experiment} for details) were used for the generalization bound in Theorem \ref{thm:pbsc} \textit{as if} the assumptions apply.} 
\label{fig:thm1_vs_depth}
\end{figure}

It is interesting take note of what the bound does \textit{not} explicitly depend on. Very often, learning bounds for linear classifiers will depend on the norm of linear classifier. But one of the most striking properties of deep networks is their ability to generalize even without penalization of weight norms. In an effort to comment on this, our bound by design does not depend on the norm of $\wpath$. Also, the architecture dimensions and input dimension, $f$ are not explicitly mentioned in the bound. This is because the effect of changing the input dimension (or architecture dimension) is captured by changes in the feature map definition, $\phi(x,w)$. In particular, the dimension of the embedding space $\PathSpace$ may change. But, just as the feature space dimension is often absent from SVM theory, the input dimension $f$ makes no appearance in our results. In fact, if the feature map $\phi(\cdot,w)$ were known \textit{a priori}, (reducing to learning to an SVM problem) then also $n$ would not appear in our bound.
 
Sample compression bounds, as in Theorem \ref{thm:pbsc}, are based on the premise that each  learned classifier is specified by some small enough subset of the training data. For example, a SVM model can always be identified by its set of $s$ support vectors. On the contrary, if $K>s$ training samples are ``memorized'' during learning, then the SVM model cannot be specified by $s<K$ samples. Suppose, \textit{a priori}, that the SVM model has at most $s$ support vectors, then there are some $m-s$ training samples on which the learned model has minimal dependence. Thus, the risk on those $m-s$ samples should approximate the true risk. This intuitively explains why specifying a DNN by means of a subset of the training data is  related to generalization.

A more general approach allows subsets of training samples to specify a sufficiently small set of $N$ models containing the learned model. The bound produced by this generalization is related to the previous $N=1$ bound by an additive factor of $\ln(N)/m$. 
Note that, for any fixed $T\in(\mathcal{X}\times\mathcal{Y})^s$, at most $2^{s+ns+n}$ different DNN classifiers, $\{x\mapsto\prednet(x):w\in\WeightSpace\}$, can simultaneously have weights $w$ that satisfy the maximum margin \Aone~ for some set of network support vectors contained in $T$.

Conceptually, there are two steps to our argument:
\begin{enumerate}
	\item{Theorem \ref{thm:skeleton} will show that for each $\wpath\in\PathSpace$, there are only $2^n$ many classifiers $\prednet$ with $\wpath=\Lambda(w)$. 
	\footnote{ Though quantitatively $2^n$ is \textit{also} the number of neuron "on"/"off" configurations, this similarity seems largely coincidental as we arrive at $2^n$ by counting allowable weight sign configurations. }
	This is true even without \Aone{}. The DNN is entirely specified by $\wpath$ modulo at most $n$ bits needed to determine weight signs.
	}
	\item{When also we have \Aone{}, we can avoid specifying $\wpath$ directly by instead supplying the image of $s$ support vectors under the feature map $\phi(\cdot,w)$. We can do this by supplying $s$ samples and at most $ns$ bits determining their image under $\sigbar(\cdot,w)$.}
\end{enumerate}

\def\sgn{sgn}
\def\hypnet{\net^{\sign}}
\begin{restatable}{theorem}{Skeleton}
\label{thm:skeleton}
For $P\subset\PathSpace$, define $\hypnet(\cdot,\Lambda^{-1}(P))\triangleq\{\hypnet(\cdot,w):w\in\WeightSpace, \Lambda(w)\in P\}$. For $\wpath\in\PathSpace$, define $\bbR^+\wpath\triangleq\{\alpha\wpath:\alpha>0\}$.

Then
\begin{equation}
|\hypnet(\cdot,\Lambda^{-1}(\bbR^+\wpath))| \leq 2^n \label{eqn:skelbound}
\end{equation}
where $n=d\width$ is the number of neurons in the Leaky-ReLU network.
\end{restatable}

Though we primarily use Theorem \ref{thm:skeleton} as a tool to prove Theorem \ref{thm:pbsc}, it has its own interesting interpretation as a characterization of the expressivity gap between SVMs and NNs, which we leave for Appendix \ref{app:skeleton_discussion}.

There are two main ideas underlying Theorem \ref{thm:skeleton} 
(Proof in Appendix \ref{app:nn_recovery}). 
Note that $\Lambda(w)$ only describes products of weights, which creates ambiguity in the scale of individual weight parameters. For example, replacing entries of $w$, $(\WeightMatrix{l+1},\W{l})$, with $(\alpha\W{l+1},\alpha^{-1}\W{l})$, does not change $\Lambda(w)$ for any choice of $\alpha>0$. This implies $|\Lambda^{(-1)}(\wpath)|=\infty$. However, the nonlinearity $\rho$ commutes with positive diagonal matrices, and class predictions are obtained as the sign of the network outputs, $\defprednet$.  Theorem \ref{thm:skeleton} implies that replacing $w$ with $\prednet$ eliminates scale information that causes ambiguity in $w$ given $\Lambda(w)$ alone. 
In other words, the set $\hypnet(\cdot,\Lambda^{-1}(\bbR^+\wpath))$ can potentially be finite as its elements cannot be indexed by a continuously-valued positive-scale parameter.

Given only $\Lambda(w)$, the second type of ambiguity in the weights $w$ is that of sign parity. Overall, $\wpath=\Lambda(w)$ forms a system of equations (one per path) involving products of the variables $\WeightMatrix{l}_{i,j}$ that cannot be solved without additional information. If  the sign of each network weight was known, we could determine the network weights by solving a system of linear equations $\Ln{|\wpath|}=\Ln{|\Lambda(w)|}$ in the variables 
$\ln( |\WeightMatrix{l}_{i,j}| )$. 
This provides a bound of $2^{P}\approx 2^{d\Omega^2}$ over the number of possibilities of $\sign(w)$, where $P$ is the number of parameters. However, this would translate to a bound governed by the ratio of the number of parameters to samples.
Such a bound is slightly unexciting in the context of deep learning, where often $P >> m$. Another idea contained in Theorem \ref{thm:skeleton} is that one can replace the number of parameters with the number of neurons. The knowledge of $\sign(\wpath)$ can be used to reduce the bound to $2^{n}=2^{d\Omega}$, where $n$ is the total number of neurons. In fact, it is an interesting intermediate result that given $\wpath$, $w$ is determined entirely by the sign of just $n$ weights in a particular geometric configuration (see Figure \ref{fig:skeleton}). (Interestingly, the \textit{sign of the weights}, which featured prominently in earlier experiments (Figure \ref{fig:wbarpos_weights}), reappears as relevant theoretical quantity). Consequently, we arrive at an improved bound governed by $n/m$.

The bound on the true risk, $\Risk$, depends on bounding the log of the number of classifiers consistent with any given training set. 
To summarize which steps in our bound over classifiers feature most prominently in our bound on $\Risk$, we tabulate the results from previous discussion in Table \ref{tab:pbsc}. As each step in our argument has an additive effect on the bound, we can speak of the "contribution of each step" to the bound on $\Risk$.

\begin{table}[h]
\caption{Additive Effect on Sample Complexity} \label{tab:pbsc}
\begin{center}
\begin{tabular}{lll}
\textbf{STEP}  &\textbf{\#WAYS} &(m-s)\textbf{$\Bound(m,n,s,\delta)$} \\
\hline \\
$\trainS\rightarrow$\nsvs   &$2^s\choosemany$   &$s\Ln{\frac{m}{s}}+O(s)$ \\
\nsvs$\rightarrow\Lambda(\WeightSpace)$ &$2^{ns}$ &$ns$\\
$\Lambda(\WeightSpace)\rightarrow\WeightSpace\rightarrow\mathcal{Y}^\mathcal{X}$ &$2^n$ &$n$\\
\end{tabular}
\end{center}
\end{table}

\subsection{On  Improvements and Further Research}
A significant reduction in the generalization bound of Theorem \ref{thm:pbsc} to well below $\mathcal{O}(ns/m)$ may be possible in practical settings. Specifically, the largest term in the numerator of the bound, $ns$, arises due to a bound over path activations on \nsvs~ that allows each sample, $x$, to choose its embedding $\sigbar_w(x)$ independently. Experimentation, however, suggests that this bound is pessimistic under practical circumstances, and training samples are instead embedded in a co-dependent manner.

To understand the dispersion of $\{\sigbar(x,w):(x,)\in\trainS\}$, we train a ReLU network with depth $d=3$ and width $\width=10$ for $50,000$ iterations on MNIST. As an output, we measure the number of unique patterns of path activation in the network, $|\{\sigbar(x,w):(x,)\in\trainS\}|$, over either training or test data as the number of training samples $m$ varied (Table \ref{tab:unique_sigbar}). For emphasis, we count $\sigbar(\samplex{i},w)\not=\sigbar(\samplex{j},w)$ as distinct patterns if even a single neuron, say $i_l$ in layer $l$, behaves differently on $\samplex{i}$ and $\samplex{j}$, i.e., $\sigl(x,w)_{i_l}\not=\sigl(x,w)_{i_l}$.

Based on previous experiments (see Figure \ref{fig:sv_vs_m}), a reasonable guess for the number of support vectors is $s=|\trainSsup|\approx0.1m$. If, in practice, for each $j\in[m]$, the embedding of the $j^\textnormal{th}$ \nsv, $\sigbar(\samplex{j},w)$, was unconstrained by that of the others, $\{\sigbar(x,w):x\in\trainSsup-\samplex{j}\}$, then with high likelihood we would expect to see around $0.1m$ unique path activations counted among support vectors. Although we do not measure this directly, we measure a relatively pessimistic upper bound instead by counting the number of unique path activations over \textit{the entire training set}. We observe that $|\{\sigbar(x,w):x\in\trainSsup\}|\leq|\{\sigbar(x,w):x\in\trainS\}|\approx 0.01m$ (Table \ref{tab:unique_sigbar}). The number of unique test embeddings of the $10k$ test samples are also relatively few (second row). This suggests that the embeddings, $x\mapsto\phi(x,w)=\sigbar(x,w)\otimes x$, over training and test data are actually  tightly coordinated, which may help further bound the number of possible embeddings of a given set of support vectors.

Future research: We recognize considerable further experimentation is needed, particularly one would like to know "under what circumstances does Assumption \ref{asm:mm_relax} hold?". We point out that to even suspect that this is an interesting question to ask requires the experimental and theoretical contributions of this paper--sometimes finding the right question is difficult in and of itself. These contributions are themselves starting points: The existence of a relationship between the number of support vectors and the architecture parameters is intriguing but warrants further exploration. And, the theoretical generalization bounds we present that depend on the number of support vectors are notable for being the only sample-compression based bounds for neural networks, but by no means do they represent the most sharpened bounds possible. Our future goal is to develop improved bounds by continuing this line of thought in the future.

\begin{figure}
\caption{Unique Sets of Active Paths Over Inputs}\label{tab:unique_sigbar}
\begin{center}
\begin{tabular}{c c c c c c}
\centering
$m=|\trainS|$ & 100 & 500 & 5000 & 20000 & 50000\\
\hline
$|\{\sigbar(x_{train})\}|$  & 49 & 75  & 210 & 282 & 711\\
$|\{\sigbar(x_{test})\}|$ & 75 & 153 & 240 & 265 & 468\\
\end{tabular}
\end{center}
\end{figure}

\section{CONCLUSION}

In this paper, we motivate and develop the study of Leaky-ReLU type deep neural networks as SVM models with embedding maps locally independent of the weights. Towards this end, we make an idealized assumption, that the neural network results in a ``max-margin" classifier. We provide an example of an experimental observation involving the configuration of the signs of the weights that is difficult to reconcile without the lens of this max-margin assumption.

Exploring the implications of this assumption, we demonstrate the experimental behavior and theoretical relevance of resulting ``network support vectors", and draw parallels between conventional support vectors and NSVs. Subsequently, we develop a generalization bound for deep neural networks that are depth-dependent in Theorem \ref{thm:pbsc}. The conceptual shift underlying the concrete ideas in the paper is to {\em parameterize} the neural network not by the weights, but as the solution to one of a small number of optimization problems.

\section{ACKNOWLEDGEMENTS}
This work was supported by the National Science Foundation and Office of Naval Research grant N000141912590 and the Army Research Office grant W911NF-19-1-0413.

\bibliography{library}
\bibliographystyle{IEEEtran}

\section{Appendix}
\subsection{Accommodation of Biases and Convolutional Layers}\label{app:bias_conv}
This section provides an interpretation of the path space and embedding map in the context of general fully-connected or convolutional Leaky-ReLU (and ReLU) networks. 
While there is a single canonical way to include biases, multiple methods may be possible for the incorporation of convolutional layers into the theory.

We turn to networks including biases. We now allow $w\in\WeightSpace$ to represent the choice of biases as well as multiplicative weights, $w=((\WeightMatrix{d},\bias{d+1}),\ldots,(\WeightMatrix{1},\bias{2}),(\WeightMatrix{0},\bias{1}))$, where each $\bias{l}\in\bbR^{\width}$ for $l\in[d]$ (and $\bias{d+1}\in\bbR$) is the bias added to the result of multiplying the activation of layer $l$ by $\WeightMatrix{l}$. Our choice of indexing is so that the $\bias{l}$ has the same dimension as the width of layer $l$. We have 

\begin{equation}
\net(x,w)\triangleq\bias{d+1}+\WeightMatrix{d}\rho(\ldots(\bias{2}+\WeightMatrix{1}\rho(\bias{1}+\WeightMatrix{0}x))\ldots).
\end{equation}

Previously, in Section \ref{sec:reparameterization}, it was established that $\net(x,w)$ could be decomposed into a sum of contributions over paths, $p=(\pathindex)$. Each path is determined by the choice of a single index per layer, including a "starting" index, $i_0\in[f]$, "connected by" a sequence of neurons, $i_l\in[\width]$ in layer $l=1,2,\ldots,d$ to the output. The contribution of this path is "seeded" with value $x_{i_0}$ and is scaled as one "moves" along the path from input to output. The scaling factor for each edge $(i_{l+1},i_{l})$ is $\WeightMatrix{l}_{(i_{l+1},i_{l})}$, which amounts to a factor of $\wpath_{\pathindex}$. The scaling factor of the $i_l^{\text{th}}$ neuron say in layer $l$ is determined by the slope of the nonlinearity of that neuron evaluated at its incoming activation during a forward pass, $\sigl(x,w)_{i_l}$. We grouped these together using the notation $\sigbar(x,w)_{(i_d,\ldots,i_1)}\triangleq \sig{d}(x,w)_{i_d} \cdots \sig{1}(x,w)_{i_1}$ .
	
With biases, the network output can still be decomposed into contributions across paths by additionally allowing paths to begin at any neuron within the network instead of only at input features:
	
\begin{align*}
\net(x,w)&=\sum_{p=(\pathindex)}\bias{d+1}+\WeightMatrix{d}_{i_d}\sig{d}(x,w)_{i_d}(\ldots\\ 
& \quad \sig{1}(x,w)_{i_1}(\bias{0}_{i_0}+\WeightMatrix{0}_{i_1,i_0}\sig{0}(x,w)_{i_0})\ldots)\\
&=\bias{d+1} + \sum_{p=(\pathindex)} \wpath_{p}\sigbar(x,w)_{i_d,\ldots,i_1} x_{i_0} \\
&\quad + \sum_{k=1}^d \sum_{p=(i_d,\ldots,i_k)} \wpath_{p} \sigbar(x,w)_{p}  \bias{k}_{i_k}
\end{align*}
The final term consists of a sum over contributions of paths--each can be interpreted as "seed value" of $\bias{k}_{i_k}$ which is then scaled by the remaining traversed edges and neurons connecting it to the output.
Note that in the above, we have augmented the definition of $\wpath$ and $\sigbar(x,w)$ by allowing additional coordinates corresponding to paths  $p=(i_d,\ldots,i_k)$ beginning at some intermediate layer ($k=1,\ldots,d$) in addition to those beginning at the input ($k=0$). We have:
\begin{align*}
\wpath_{i_d,\ldots,i_k}\triangleq &\WeightMatrix{d}_{i_d}\cdots \WeightMatrix{k}_{i_{k+1},i_k}\\
\sigbar(x,w)_{i_d,\ldots,i_k}\triangleq &\sig{d}(x,w)_{i_d} \cdots \sig{k}(x,w)_{i_k}.
\end{align*}
with the convention that $\sig{0}(x,w)_{i_0}=x_{i_0}$. To round out the notation, if we define a "dummy bias", $\bias{0}$, to be a vector of all ones, $\forall i_0$ $\bias{0}_{i_0}=1$, then we get a clean formulation for the network output:
\begin{align*}
\net(x,w)&=\bias{d+1}+\sum_{\substack{p=(i_d,\ldots,i_k)\\ k=0,\ldots,d}} \wpath_p \sigbar(x,w)_p \bias{k}_{i_k}\\
&\triangleq\bias{d+1}+\langle \wpath , \phi(x,w) \rangle 
\end{align*}

Turning now to convolutional layers, we seek a simple modification that will allow an analogous max-margin formulation. Consider a network consisting of several convolution layers parameterized by $w^{\text{conv}}$, followed by fully-connected layers parameterized by $w$. To generalize, we simply replace treat the convolution embedding of the inputs $\Psi^{\text{conv}}(x,w^{\text{conv}})$ as if they were the inputs themselves within the SVM formulation:
\begin{equation*}
\net(x,w, w^{\text{conv}}) \triangleq \bias{d+1}+\langle \wpath, \phi(  \Psi^{\text{conv}}(x,w^{\text{conv}}), w)
\end{equation*}
Given that we expect the initial convolutional layers to quickly arrive at certain edge-detecting low level filters that are generically useful, this treatment of the convolutional output as if it were a fixed input may be somewhat justifiable. Most importantly, this simple modification does in fact yield experimental results for convolutional networks that are similar to those we find for fully-connected. 

\subsection{Relevance of the Max-Margin Assumption}\label{app:Aone}

The value of an assumption is in its implications and relevance. If theoretical work in this paper shows the former, this section is aimed at demonstrating the later. There is a bit of nuance in that "relevance" is to be distinguished from "validity". That is, \Aone{} is not a "conjecture". It is not something that we are supposing applies exactly to unregularized deep learning models as they are. That is unclear. However, this section will show that empirically, trained deep network models and their max-margin counterparts behave extremely similarly.

What then is the value of analyzing max-margin networks without first establishing the validity of \Aone{} theoretically? It turns out that analyzing the consequences of \Aone{} is easier than establishing its validity (if true). Furthermore, it is useful to know ahead of time that \Aone{} has theoretical consequences before undertaking the task of trying to prove it. Such a study should require additional assumptions about the training data and the initialization, and it is not clear at this time what those should be. 

Secondly, we should not fall into the trap of thinking of deep learning as a fixed phenomenon for observational study only. 
As engineers trying to build better models, we can make it as we like. If it turns out that \Aone{} is not yet strictly speaking true but has interesting theoretical consequences, then we may modify the training process so that the trained network \textit{is} a max-margin network. It is also not clear right now what the best way to do that is. Though as we shall see, these max-margin models would not represent a huge divergence from current deep learning models. Instead our results indicate the two are quite similar.

Consider a comparison of the two functions $\prednet(\cdot,w)$ and the associated max-margin classifier on the same data with feature map $\phi(\cdot,w))$. We compare these functions by comparing the value they return on a finite set of inputs using one of two strategies. The first approach, taken in Figure \ref{fig:vis_compare}, is to train each on input data that is merely $2$ dimensional so that the decision boundary can actually be visualized by evaluating on a grid of input points. The second approach, taken in Table \ref{tab:val_compare}, is to use more realistic input data for training, such as CIFAR-10, but to compare outputs on validation data instead. Though this will not imply that the two functions are equal everywhere, if we are interested using the max-margin assumption for generalization theory, then high probability agreement on support of the data distribution is sufficient.

For the first approach in Figure \ref{fig:vis_compare}, we designed $3$ toy data sets and trained a fixed $9$ layer fully-connected (FC) network on each of them, obtaining weights $w$ and classifier $\prednet(\cdot,w)$. Then we used the scikit-learn library
to train a max-margin linear classifier on the image of the same training data under the embedding map $\phi(\cdot,w)$ for the same weights $w$. More details available in the appendix \ref{app:experiment}. Optically, the decision boundaries of the DNNs the left column \ref{vis:col-left} trained by back propagation and their max-margin counter parts in the right column \ref{vis:col-right} are quite similar. Where the decisions of the two classifiers differ, the data samples with very low probability, suggesting that the two have very similar generalization error.

\begin{figure}
	\centering
	\def\nsbwid{0.4\columnwidth}
	\def\nsbht{\columnwidth}
	\def\negvert{-0.5cm}
	\begin{subfigure}[b]{\nsbwid}
			\centering
			\includegraphics[width=\linewidth]{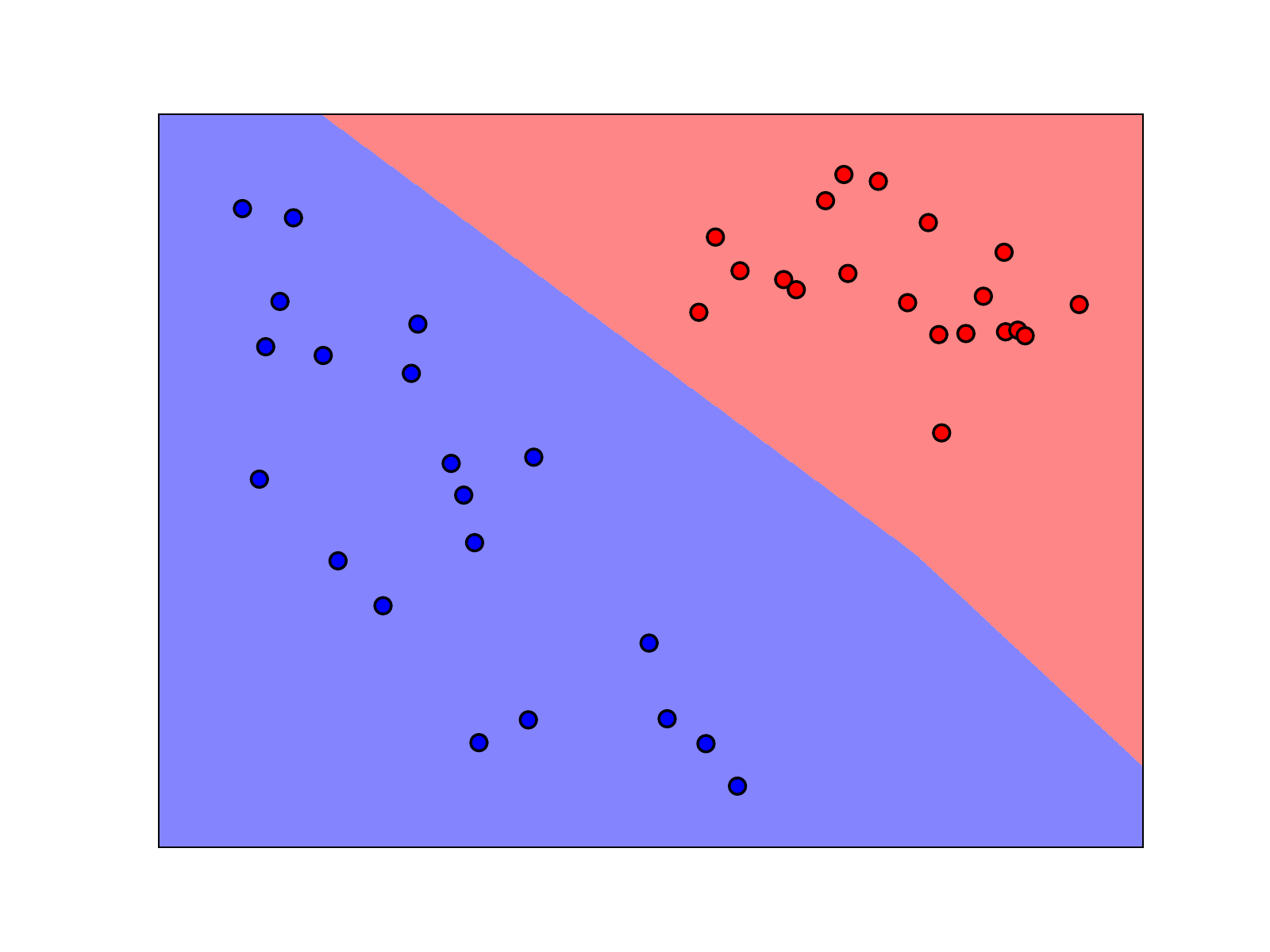}
	\end{subfigure}
	\hfill
	\begin{subfigure}[b]{\nsbwid}
			\centering
			\includegraphics[width=\linewidth]{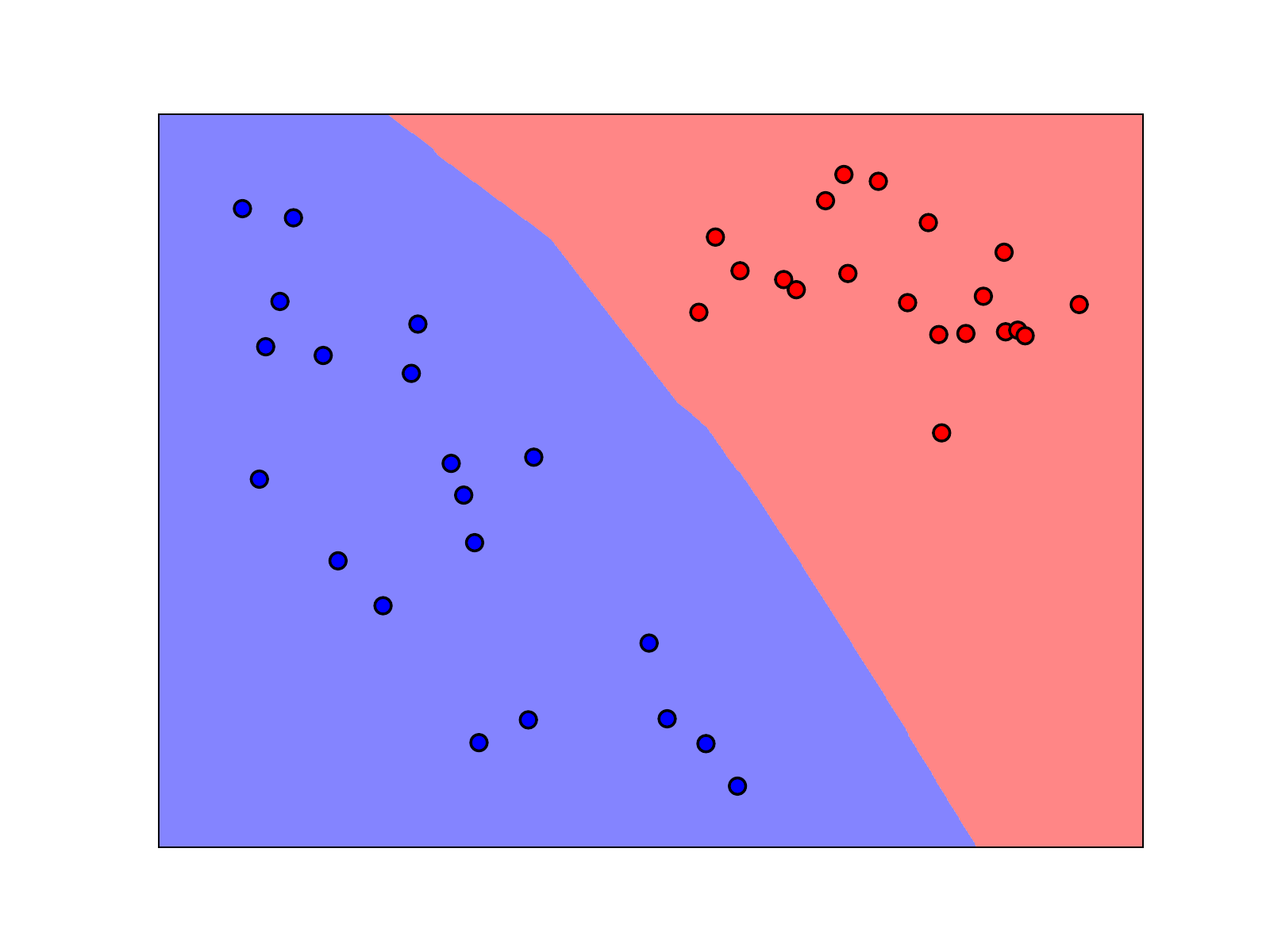}
	\end{subfigure}
	\\[\negvert{}]	
	\begin{subfigure}[b]{\nsbwid}
		\centering
		\includegraphics[width=\linewidth]{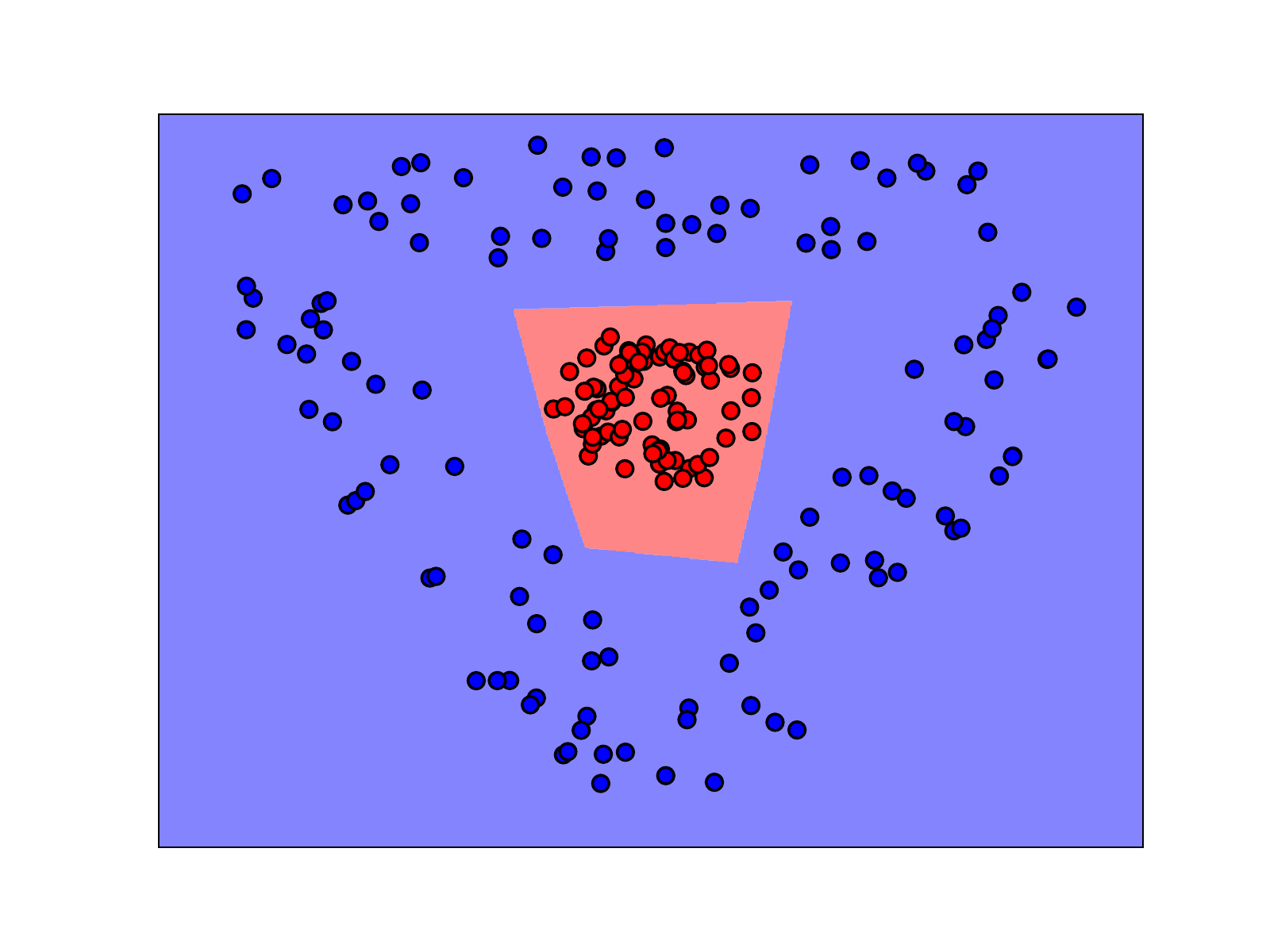}
	\end{subfigure}
		\hfill
	\begin{subfigure}[b]{\nsbwid}
		\centering
		\includegraphics[width=\linewidth]{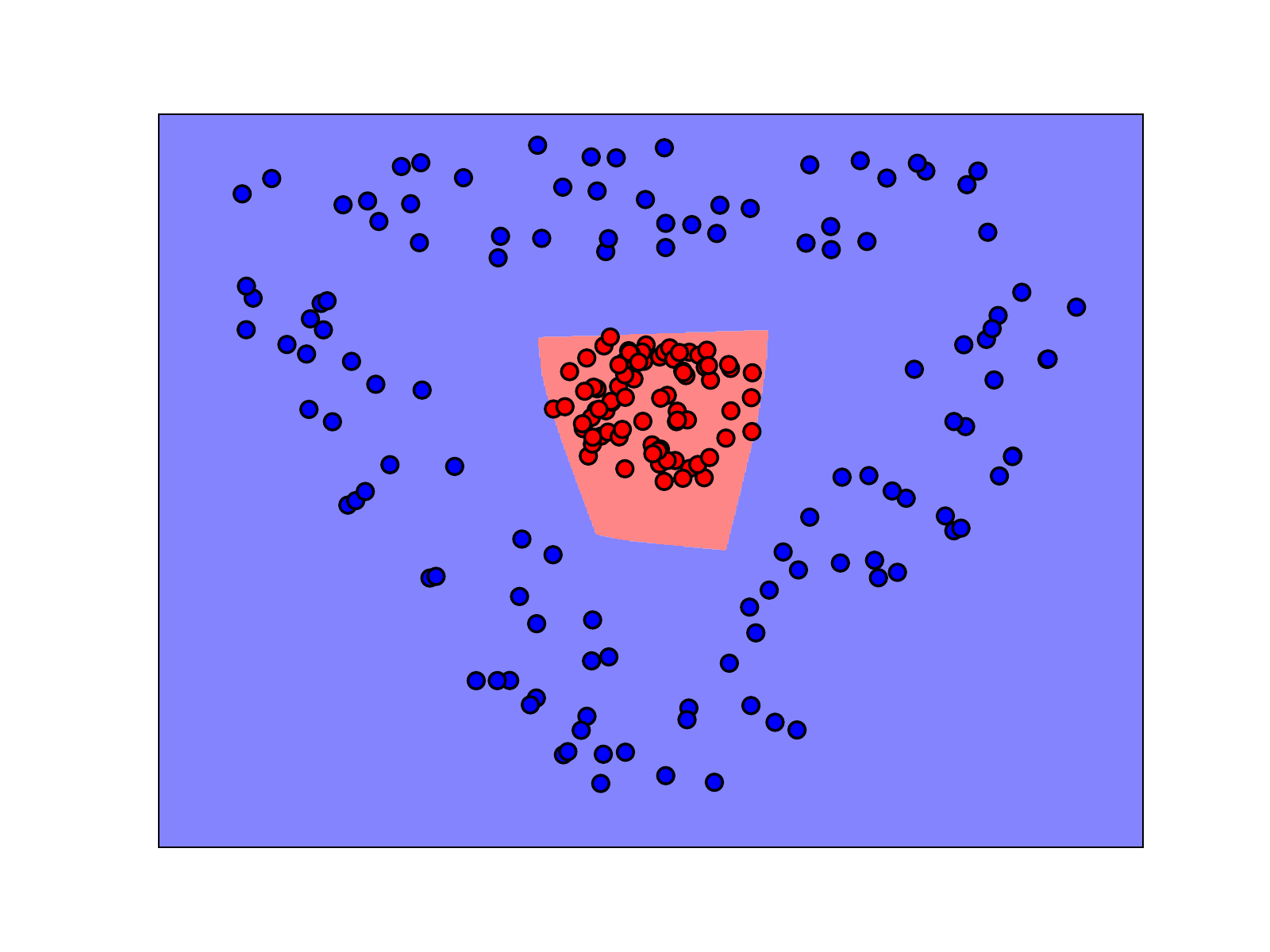}
	\end{subfigure}\\[\negvert{}]
	\begin{subfigure}[b]{\nsbwid{}}
	\centering
	\includegraphics[width=\linewidth]{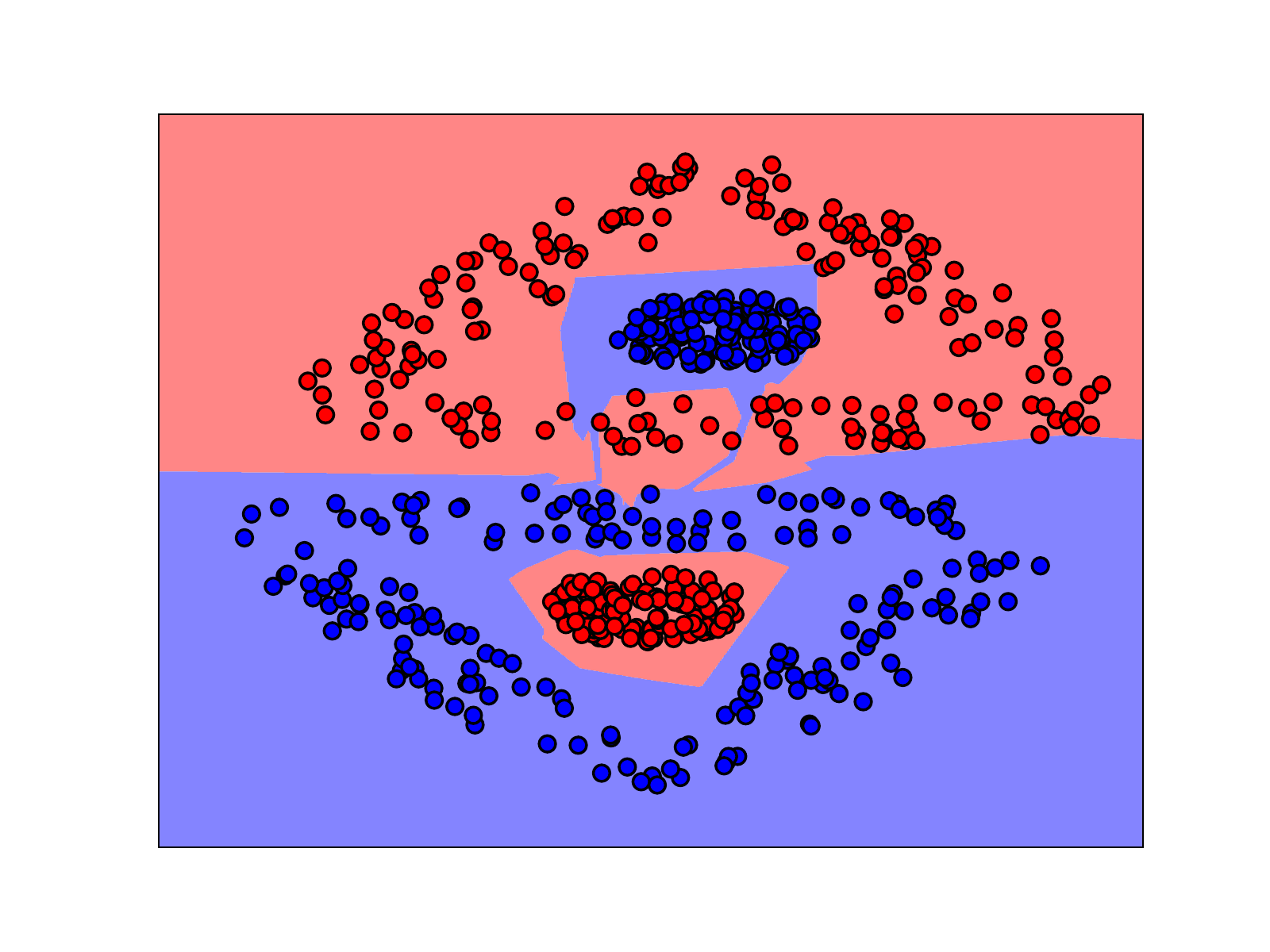}
	\subcaption{ Learned DNN Classifier}
	 \label{vis:col-left}
	\end{subfigure}
	\hfill
	\begin{subfigure}[b]{\nsbwid}
		\centering
		\includegraphics[width=\linewidth]{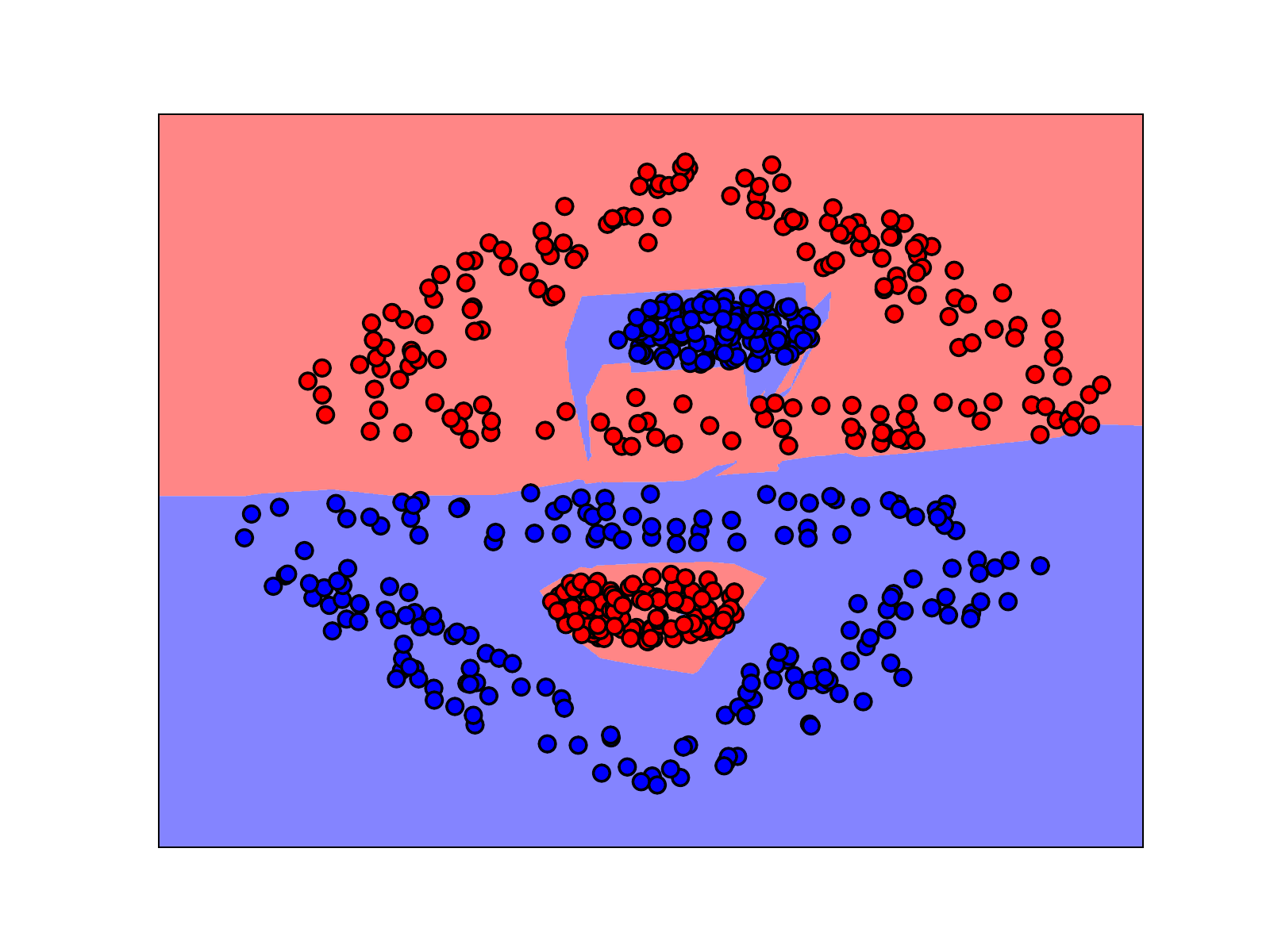}
		\subcaption{  Max-Margin Classifier}
		 \label{vis:col-right}
	\end{subfigure}
	\caption{Displaced in each row is a  visual comparisons between the DNN classifier(Column \ref{vis:col-left}) and the max-margin classifier(Column \ref{vis:col-right}). Each row corresponds to a different training dataset, and each image corresponds to a different classifier. Blue[red] circles represent training data with positive[negative] labels.  Blue[red] regions in a particular image represent inputs assigned positive[negative] label by the corresponding classifier. The DNN classifiers, $\prednet$ in the left column are obtained by gradient descent on the displayed training data, $\trainset$, to learn a set of weights $w$. These weights are fixed for the entire row and define our feature embedding $\phi(\cdot,w)$. This embedding map is used to train a max-margin classifier on the training data $\trainembed$, which is displayed in the right column. Further experimental details can be found in Section \ref{app:experiment}. Although a $2$ dimensional input space is far from a general setup, at least in this setting that we are able to visualize the max-margin classifier and $\prednet$ models appear visually very similar. Where there are differences in the decision boundary, those differences do not appear in the vicinity of the training data (and perhaps, the data distribution).}
	\label{fig:vis_compare}
\end{figure}

In the second approach, we classify Frogs vs Ships on a binarized CIFAR-10 dataset using a convolutional layer network. We varied the number of fully-connected(FC) layers following the initial $5$ layers of alternating convolution and max pooling. The range of depths we chose was determined by technical constraints and more details are available in \ref{app:experiment}. By design, the DNN had perfect training accuracy, and therefore, so did the max-margin classifier. When we compared the two classifiers across a range of depths (Table \ref{tab:val_compare}), not only were the validation accuracies very similar, but also the \textit{same} samples were misclassified by both models, whose predictions agreed more than $99\%$ of the time. 

\begin{table}[h]
	\def\nconv{5}
	\newcommand{\lyrnum}[1]{\nconv+#1}
	\caption{DNN vs Max-Margin on CIFAR Validation Data}\label{tab:val_compare}
	\begin{center}
		\begin{tabular}{cccc}
			\centering
			\textbf{(FC) Depth}  &\textbf{Score Net} &\textbf{Score Max-Margin} &\textbf{Prob Agreement} \\
			\hline \\
			3& 0.968 & 0.964 & \textbf{0.990}\\
			4& 0.976 & 0.973 & \textbf{0.992}\\
			5& 0.972 & 0.972 & \textbf{0.993}\\
			6& 0.974 & 0.974 & \textbf{0.995}\\ 
			7& 0.971 & 0.969 & \textbf{0.994}\\
		\end{tabular}
	\end{center}
\end{table}

We can make a final, clever attempt to empirically invalidate \Aone{}. (One can never empirically \textit{validate} a hypothesis). If the max-margin assumption were actually true, what else would we expect to see? We seek to exploit the fact that for every $x$, $\sigbar(x,w)$ is coordinate-wise positive, since each entry is a product of $\beta=0.1$ and $\gamma=1.0$, each raised to various powers that depend on $x$ and $w$. We observe that when $\sampleyj \samplexj$ is also coordinate-wise positive for each training sample, $(\samplexj,\sampleyj)$, so too is each embedded, $\sampley{j}\phi(\samplex{j},w)$. Suddenly, we have a seemingly strong conclusion: since the max-margin classifier is in the positive linear combination of the $\{\sampley{j}\phi(\samplex{j},w)\}_{j=1}^m$, we see that \Aone{} implies that $\wpath$ is coordinate-wise positive.

Yet, experimentally we can reproduce this theoretical implication. We consider training data $\trainS\subset \bbR^2$ organized by label into the $1^{st}$ and $3^{rd}$ quadrants so that for each training datum $(\samplex{j},\sampley{j})$ is coordinate-wise positive. The weights obtained from training (without biases) with Leaky-ReLU nonlinearity on the described data are shown graphically in Figure \ref{fig:wbarpos_weights} (More details and training data can be found in appendix. The decision boundary of this network is shown in Figure \ref{fig:wbarpos_dataclassif}.

\begin{figure}[h]
	\centering
	\includegraphics[width=\linewidth]{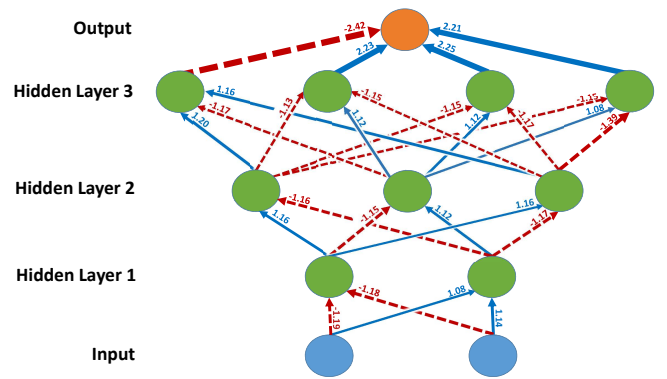}
	\caption{Learned network weights after training on data with positive samples in the $1^{st}$ quadrant and negative samples in the $3^{rd}$ quadrant. Negative [positive] weights are represented by red dotted [blue solid] arrows [respectively]. Thicker arrows correspond to weights of larger magnitude. The finding is that each path from any input feature to the output contains an even number of red arrows (negative weights). This coordination of weight signs across layers is a striking feature of training that is implied by Assumption \ref{asm:mm_relax}, but is not readily explained otherwise. 
	}
	\label{fig:wbarpos_weights}
\end{figure}

\begin{figure}[h]
	\centering
	\includegraphics[width=\linewidth]{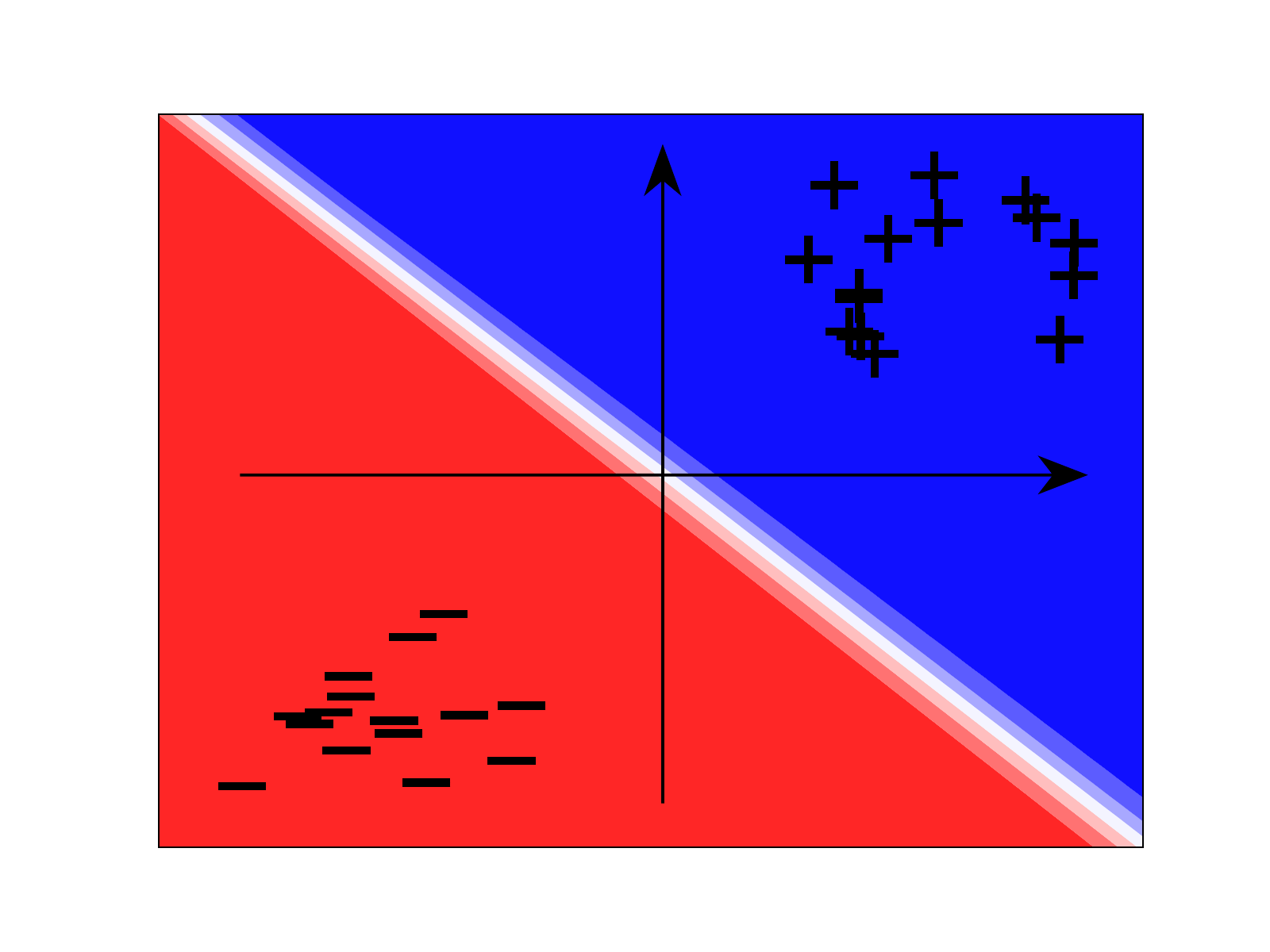}
	\caption{Learned decision boundary and training data corresponding to the learned weights in Figure \ref{fig:wbarpos_weights}. "Black plus [minus] signs correspond to locations of positively [negatively] labeled training data. Blue [red] regions correspond to positive [negative] evaluations by the network.}
	\label{fig:wbarpos_dataclassif}
\end{figure}

Through close inspection of Figure \ref{fig:wbarpos_weights}, we see that every path from the input to output traverses an even number of negative weights, which is of course equivalent to $\wpath$ being coordinate-wise positive. Not only is this an immediate consequence of Assumption \ref{asm:mm_relax}, it is not clear to us through any other theoretical lens that we are aware of. Notice also that this regularizing structure of the weights is not nearly so apparent when the weights are conceptually grouped by layer instead of across paths.

\subsection{Experiment Details}\label{app:experiment}
Concerning the experiment described in Figures \ref{fig:wbarpos_weights} and \ref{fig:wbarpos_dataclassif}, 
we use a truncated normal weight initialization centered around $0$ with $0.025$ standard deviation. We train with gradient descent for $15000$ iterations with a learning rate of $0.005$. Our nonlinearity, Leaky-ReLU, has slopes $\beta=0.1$ and $\gamma=1.0$.

The primary finding from the experiment, $\wpath> 0$, happens reliably as long as the weight initialization and learning rate are suitably small. Just as we are not claiming Assumption \ref{asm:mm_relax} always holds, we are also not claiming that $\wpath>0$ always holds exactly under all related circumstances. For example, if the weight initialization is too large, it is possible to have some few very small weights with signs that do not agree with $\wpath>0$, though the entries of $\wpath$ with largest magnitude will all have the same positive sign. Optically, it seems like the gradient can become small too quickly to overcome a large initialization of a given weight with the "wrong" sign. Though, a complete analysis of this phenomenon is not part of the scope of this work.

For Figures \ref{fig:sv_vs_m}, 
\ref{fig:sv_vs_width}, and 
\ref{fig:sv_vs_depth}, the setup is slightly different. 
We train a fully connected neural network with nonlinearity $\rho(x)=ReLU(x)$ ($ReLU(x)=\max\{0,x\}$) using SGD with momentum parameter $0.05$, learning rate $0.01$, and batch size $100$. We train on "flattened" MNIST images ($f=28\times28$) with labels grouped into the binary classes $\{0,1,2,3,4\}$ and $\{5,6,7,8,9\}$. Unless explicitly varied in the figure, we use a fixed, random subset of $m=20000$ training samples and an architecture consisting of $d=3$ hidden layers of uniform width, $\width=16$. All experiments displayed actually achieved $0$ training error. The reason we use only $2/5^{ths}$ of the training data is because  achieving \textit{exactly} $0$ training error with every architecture considered is necessary to compare the number of support vectors and difficult to do with the entire training set.

Once we train the network to learn $w$, we experimentally determine the set of network support vectors by running a SVM classifier on the embedded data defined by the \textit{fixed} feature map $x\mapsto\phi(x,w)$. To match the constraints.svm.SVC function in  the scikit-learn library, we use hinge loss and regularization constant $C=1e^{-5}$. We argue though that when the training error is identically $0$ and the data is linearly separable, the SVC model with hinge loss will return the maximum margin classifier independently of the value of $C$. This is because for any $C$, the weights are eventually near the optimum where none of the constraints are active. This agrees with what we see experimentally when we varied $C$ (not shown).

The data points in Figures \ref{fig:sv_vs_width} and \ref{fig:sv_vs_depth} representing the number of support vectors vs width and depth are all averages of $3$ trials. One tricky experimental detail is that neural network models have to be trained for a very long time, sometimes upwards of $100$ epochs, in order to get \textit{exactly} $0$ training error needed to guarantee linear separability. This is especially true for the larger width and larger depth runs.

When we randomize the labels, as in Figure \ref{fig:noisy_label_sv_vs_m}, we are determining \textit{every} sample label by a fair coin flip once before training starts, then fixing that label during training.

For the comparison of max-margin classifier in Figure \ref{fig:vis_compare}, three different synthetic datasets were generated by random sampling. The idea in the choice of distributions was to give a variety of both "easy" and "difficulty" 2-dimensional classification tasks. The network used was a $9$ hidden layer fully connected network of widths $4,6,8,10,12,14,16,20,30$. The training parameters  used to obtain weights $w$ were identical to above. 

To produce the max-margin classifiers, we used $\phi(\cdot,w)$ as a fixed embedding (corresponding to the learned parameters $w$, and trained a max-margin classifier using the sci-kit learn library. Specifically, we first calculated the kernel matrix for all training samples. Then we trained a linear classifier using $C=1e-5$ and tolerance $1e-5$ without the shrinking heuristic available to the SVC solver.

For the convolutional experiments in Figure \ref{fig:cifar_sv_vs_depth} and Table \ref{tab:val_compare}, we used a convolutional network. The first $5$ layers consisted of $3$ convolutional layers of $64$ filters each, interleaved by $2$ max pooling layers. The convolutional layers used $3x3$ kernels with a stride of $1$, and the max pooling layers took the maximum over $2x2$ regions. This convolutional embedding was flattened. Experimentally, we varied "FC depth", or the number of subsequent fully-connected $64$ neuron layers between this flattened output and the network output.

The dataset, CIFAR-10, was chosen based on suggestion by a reviewer. Because we only study binary classifiers, we restricted ourselves to discriminating "Frogs" from "Ships". This was also simply the first binarization that we tried. A foreseen benefit was also that there would be only $10k$ training samples had either of these labels, which makes running in-memory SVM classification easier.

There were upper and lower constraints on the ranges of the FC depth explored. On the higher end, we found that it was impossible to use a \textit{fixed} learning rate of $0.005$ across a huge range of depths. For large FC depth, the training would become unstable in a way that could be mitigated by decreasing the learning rate.

The lower constraint limiting FC depth $\geq 3$ is slightly curious. Though the network would still have perfect training accuracy, the SVM solver would struggle to find any linear separator. We know theoretically that one must exist (since $\wpath$ is one), but it seems in practice that the SVM solver has trouble finding it for shallow networks.

\def\Skel{Skel}
\def\SkelW{SkelW}
\begin{figure}
\centering
\includegraphics[width=\linewidth]{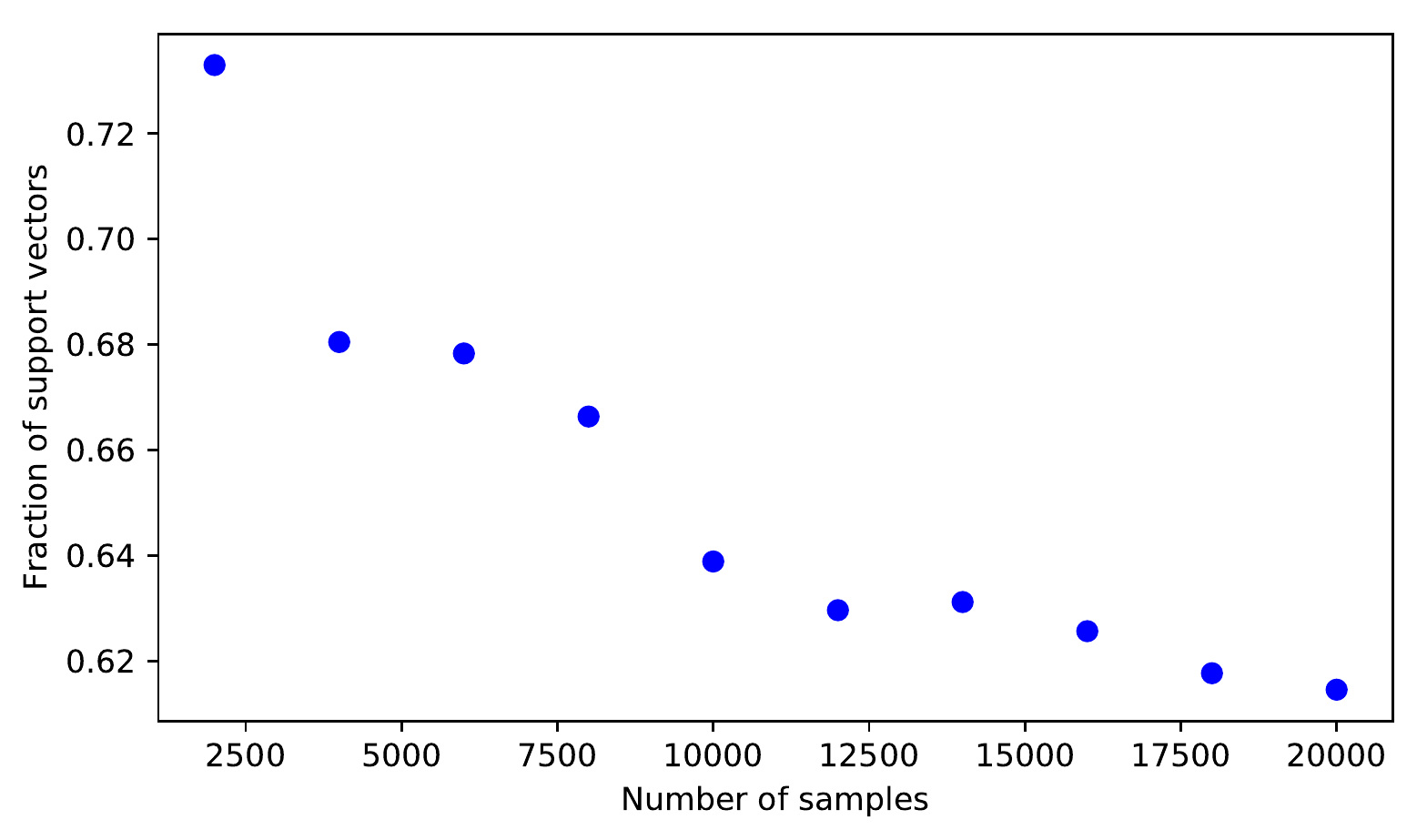}
\caption{Fraction Network Support Vectors $(s/m)$ vs $m$ under Randomized Labels: Once before training, the label of each training datum is replaced by a sample drawn uniformly from $\mathcal{Y}$. Compared to the setting with true labels (Figure \ref{fig:sv_vs_m}), the data appear shifted up vertically by $~0.5$.}
\label{fig:noisy_label_sv_vs_m}
\end{figure}

\subsection{Theorem \ref{thm:skeleton}: The Skeleton and NN Recovery}\label{app:nn_recovery}

\begin{figure}[ht!]
\centering
{\includegraphics[width=0.8\linewidth]
{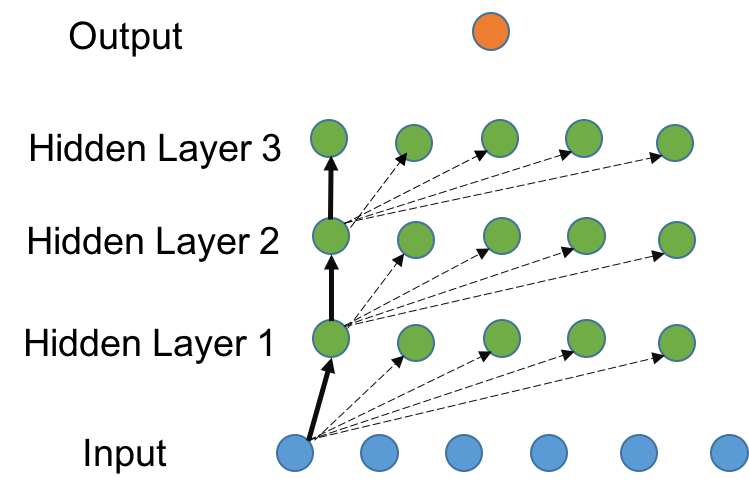}}\\
\caption{ An illustration of one possible collection of edges corresponding to a skeleton (the key ingredient in the proof of Theorem \ref{thm:skeleton}). A "skeleton" is a collection of $n$ 
edges, $\Skel$, with corresponding network weights, $\SkelW$, containing for each neuron one path from some input feature to that neuron. For each $\wpath$, $\SkelW$ is in bijection with $\Lambda^{-1}(\wpath)$. We may imagine the solid black lines to be the "spine" and the dotted lines to be the "ribs", though there are valid configurations that are less anatomic.}
\label{fig:skeleton}
\end{figure}

\Skeleton*

\begin{proof}

Suppose we are given a positive multiple of $\wpath$. We may assume without loss of generality that every neuron $\eta$ belongs to at least some path, $p(\eta)$, with $\wpath_{p(\eta)}\not=0$. If some neuron is not a member of any such path, then it makes no contribution to the function $x\mapsto\net(x,w)=\sum_{p\in Paths}\wpath_p \phi(x,w)_p$. Thus we may drop all such neurons without affecting which functions $\prednet$ are feasible given $\wpath/\norm{\wpath}$. If by removing neurons in this manner we run out of neurons in a single hidden layer, then the bound is trivially true since $\net_w$ must be the zero function.

Thus, for all $1\leq l\leq d$, every neuron $i_1$ in layer $1$ has at least some corresponding index $s_0(i_1)$ in layer $0$ such that $\W{0}_{i_1,s_0(i_1)}\not=0$. Because Leaky-ReLU commutes with positive diagonal matrices, we can rescale column $i_1$ of $\W{1}$ by $|\W{0}_{i_1,s_0(i_1)}|$ and row $s_0(i_1)$ of $\W{0}$ by $|\W{0}_{i_1,s_0(i_1)}|^{-1}$.

We continue renormalizing this way until every row of each of the weight matrices $\W{0}$ through $\W{d-1}$ has at least one weight in $\{-1,1\}$. For each neuron not corresponding to the input or output, fix a particular choice of indices in the previous layer,
$\Skel\triangleq\{(l,i_l,s_{l-1})\}_{l,i_l}$, so that the corresponding weights $\SkelW\triangleq\{\W{l-1}_{i_l,s_{l-1}} \}_{l,i_l}$ are all in $\{-1,1\}$. Call these indices, $\Skel$, a "skeleton" of the network, and the corresponding weights $\SkelW\subset\{-1,1\}^n$ "skeleton weights". A path $p=(i_d,\ldots,i_0)$ will be said to be "in the skeleton" if $\forall l$ $\W{l}_{i_{l+1},i_l}$ is a skeleton weight. We have shown that as long as every neuron is along some path $p$ with $\wpath_p\not=0$, then the network has a skeleton. (Figure \ref{fig:skeleton} illustrates one such configuration of weights, but is not explicitly used in this proof).

We will show that given $\alpha\wpath$ and a skeleton $Skel$, every choice of skeleton weights determines a different set of weights $w=(\alpha\W{d},\W{d-1},\ldots,\W{0})$ with $\Lambda(w)=\alpha\wpath$. Thus we will show that the set of weights compatible with $\wpath$ are in bijection with the set of $2^n$ possible skeleton weights, up to rescaling of $\W{d}$. Since scaling $\W{d}$ by $\alpha>0$ doesn't change the sign of the network output, we will have at most $2^n$ distinct possible classification functions compatible with some scaling of $\wpath$.

\def\bpath{\bar{b}}
Fix a choice of skeleton $\Skel$ and skeleton weights $\SkelW$. Then for every layer $l$, for every neuron $i_l$ in layer $l$, there is a path $(j_{l-1},j_{l-2},\ldots,j_0)$ from the input to that neuron which stays in the skeleton and for which the product of weights is $\pm1\not=0$. We introduce the notation $\bpath$ by using $\bpath_{j_{l-1},j_{l-2},\ldots,j_0}$ to mean "the product of weights along a particular path within $\Skel$ from the input to a particular neuron":
\begin{equation}
\bpath_{i_l,j_{l-1},j_{l-2},\ldots,j_0}\triangleq \W{l-1}_{i_l,j_{l-1}}\W{l-2}_{j_{l-1},j_{l-2}}\cdots\W{0}_{j_1,j_0}.
\end{equation}

First we show that all of the projection weights, $\W{d}$, are determined up to scale by $\alpha$. For neuron $i_d$ in layer $d$, get a path $j_{d-1},j_{d-2},\ldots,j_0$ from the input to neuron $i_d$ within the skeleton so that $\bpath_{j_{d-1},j_{d-2},\ldots,j_0}\not=0$. Then simply solve
\begin{equation}
\frac{\alpha\wpath_{i_d,j_{d-1},j_{d-2},\ldots,j_0}}{\bpath_{i_d,j_{d-1},j_{d-2},\ldots,j_0}}
=\frac{\alpha\W{d}_{i_d}\W{d-1}_{i_d,j_{d-1}}\cdots\W{0}_{j_1,j_0}}{\W{d-1}_{i_d,j_{d-1}}\cdots\W{0}_{j_1,j_0}}
=\alpha\W{d}_{i_d}.
\end{equation}

We next show how to find any weight. let $l<d$,$i_{l+1},i_l$ be arbitrary. 
From neuron $i_{l+1}$ in layer $l+1$ and neuron $i_l$ in layer $l$ get paths $(i_{l+1},j_{l},j_{l-1},\ldots,j_0)$ and $(i_l,k_{l-1},\ldots,k_0)$ within $\Skel$ with $\bpath_{i_{l+1},j_{l},j_{l-1},\ldots,j_0}$ and $\bpath_{i_l,k_{l-1},\ldots,k_0}$ nonzero. Furthermore, we are guaranteed some indices $e_d,e_{d-1},\ldots,e_{l+2}$ such that $\wpath_{e_d,e_{d-1},\ldots,e_{l+2},i_{l+1},j_{l},j_{l-1},\ldots, j_0}\not=0$ since every neuron is connected to the output through at least some path with nonzero weights. Then we simply solve

\begin{align}
&\frac{\alpha\wpath_{e_d,e_{d-1},\ldots,i_{l+1},i_{l}, k_{l-1},\ldots,k_0}}{\bpath_{i_{l}, k_{l-1},\ldots,k_0}}
\frac{\bpath_{i_{l+1},j_{l},j_{l-1},\ldots,j_0}}		{\alpha\wpath_{  e_d,e_{d-1},\ldots,i_{l+1},j_{l},j_{l-1},\ldots,j_0   }}\\
=&\frac{\alpha\W{d}_{e_d}\W{d-1}_{e_d,e_{d-1}}\cdots\W{l+1}_{e_{l+2},i_{l+1}}\W{l}_{i_{l+1},i_l}}{\alpha\W{d}_{e_d}\W{d-1}_{e_d,e_{d-1}}\cdots\W{l+1}_{e_{l		+2},i_{l+1}}}\\
=&\W{l}_{i_{l+1},i_l}.
\end{align}

\end{proof}

\subsection{Broader Significance Theorem \ref{thm:skeleton} Discussion} \label{app:skeleton_discussion}
While, numerically, $n$ is equal to the number of neurons, in this setting one should think of it as the smallest number of weighted edges in the network graph needed to connect every neuron to the output. Such a subset, we called a skeleton. Similarly, $2^n$ refers not to the number of neuron state configurations, but to the number of sign configurations of the n weights whose edges are in the skeleton.

One nice perspective afforded by the exposition in our paper is that while a Support Vector Machine(SVM) learns a classifier in a feature space, a Neural Network(NN) learns both a classifier, $\Lambda(w)$, and an embedding map, $\phi(x,w)$. The flexibility to learn the embedding can then be seen as an additional mode of expressivity, loosely speaking of course, that is available to NNs but not to SVMs. However, the extent of this flexibility is unclear because both the classifier, $\Lambda(w)$, and the embedding map, $\phi(x,w)$, depend on the weights, i.e., they are entangled. What is the nature of this dependence?

Theorem \ref{thm:skeleton} answers that question completely. Fix any skeleton subgraph. It says each classifier, $\Lambda(w)$, corresponds to only finitely many embedding maps, with one map corresponding to each one of the $2^n$ configurations of weight-signs in that skeleton.

\subsection{PAC-Bayes Background}
In this section, we review the sample compression version of PAC-Bayes bounds, which we will invoke to prove Theorem \ref{thm:pbsc}. We are largely following \cite{Laviolette2007}.

In the PAC-Bayes framework (without sample compression), one typically works with a distribution over classifiers that is updated after seeing the training set $\trainS$. A prior $P$ over $\HypSpace\subset\mathcal{Y}^\mathcal{X}$ is declared before training, without reference to the specific samples in $\trainS$. Then consider another distribution over classifiers, $Q$, called a "posterior" to reflect that it is allowed to depend on $\trainS$. Each posterior $Q$ defines a Gibbs classifier $G_Q$ that makes predictions stochastically by sampling classifiers according to $Q$. Similarly, we define the true risk $R(G_Q)$ and empirical risk $R_S(G_Q)$ of the Gibbs classifier $G_Q$ as

\begin{align}
    &R_{\pdata}(G_Q) = \mathop{\mathbb{E}}_{h\sim Q}[R_{\pdata}(h)]   
    &R_{\trainS}(G_Q) = \mathop{\mathbb{E}}_{h\sim Q}[R_{\trainS}(h)] \nonumber
\end{align}

PAC-Bayes gives a very elegant characterization of the relationship between the true and empirical risks of Gibbs classifiers. Let $KL(Q||P)$ be the Kullback-Leibler divergence between distributions $Q$ and $P$. For scalars $q,p$, define $kl(q||p)$ to be the Kullback-Leibler divergence between Bernoulli $q$ and $p$ distributions. We have the following classical uniform bound over posteriors $Q$ $\forall\delta\in(0,1]$ (Theorem 1 in \cite{Laviolette2007})
\begin{align}
&\Prob{\forall Q\quad \Phi(Q,P,m,\delta) }\geq 1-\delta \nonumber\\
\intertext{where $\Phi(Q,P,m,\delta)$ is the event}
&kl(\risk{\trainS}{G_Q}||\risk{\pdata}{G_Q}) \leq\frac{KL(Q||P) + \ln\frac{m+1}{\delta}}{m}
\end{align}

To relate this to classical bounds 
for finite hypothesis sets, notice that if $P=Unif(\HypSpace)$ and $Q$ is a delta distribution on $\hyp_0\in\HypSpace$, then the PAC-Bayes bound is governed by the ratio $KL(Q||P)=\ln|\HypSpace|$ to $m$. When $\HypSpace$ is not finite, one can still get bounds for stochastic neural network classifiers as in \cite{Dziugaite2017}, or one can convert these bounds into bounds for deterministic classifiers by considering the risk of the classifier which outputs the majority vote over $Q$ (see \cite{Laviolette2007} section 3 for example). We take neither of these approaches, but just mention them for the reader's interest.

Thus far we have discussed "data-independent" priors. We now turn to \cite{Laviolette2007} to discuss priors $P_{\trainS}$ over hypotheses that depend on the training set through a "reconstruction function", $\Ron$, mapping subsets of the training data and some "side-information" to a hypothesis.

The idea is to describe classifiers in terms of a subset of training samples, called a "compression sequence", and an element from some auxiliary set, called a "message". For the moment, consider any arbitrary sequence\footnote{The terminology "sequence" is used here to highlight situations which apply to \textit{any} set of inputs, not just probable ones.} , $T\subset(\mathcal{X}\times\mathcal{Y})^m$. Given $T$, define a set of "allowable messages" $\isp(T)$ so that we have a "reconstruction function", $R:T\times\isp(T)\mapsto \mathcal{Y}^\mathcal{X}$, which sends arbitrary sets of samples $T$, and optionally some side information in $\isp(T)$, to a classifier mapping $\mathcal{X}$ to $\mathcal{Y}$.

Let $I$ be the set of subsets of $[m]$. Considering now our training set $\trainS$, for $\idx\in I$ define $\tset_\idx$ to be the subset of training points at indices $\idx$. 
We now introduce a single (data-dependent) set for the support of our prior and posterior. Define

\begin{equation}
\isp_{\tset}\triangleq \bigcup_{\idx\in I} \isp(\tset_{\idx}).
\end{equation}

\def\domain{I\times\isp_{\tset}}
In the sample compression setting, we sample hypotheses in $\mathcal{Y}^\mathcal{X}$ by sampling $(\idx,z)$ from $\domain$ according to either our ($\tset$-dependent) prior $P_{\tset}(\idx,z)$ or our posterior $Q(\idx,z)$ and passing $(\idx,z)$ to our reconstruction function $R$ to obtain the hypothesis $R(\idx,z):\mathcal{X}\mapsto\mathcal{Y}$.

For the results to follow, we require our prior and posterior to factorize accordingly:

\begin{align}
&P_{\tset}(\idx,z)=P_I(\idx)P_{\isp(\tset_{\idx})}(z) \nonumber\\
&Q(\idx,z)=Q_I(\idx)Q_{\isp(\tset_{\idx})}(z) \nonumber
\end{align}

That is, though the prior does depend on the training set, the marginal prior $P_I$ over subsets $\idx\in I$ does \textit{not}. Also, conditioned on $\idx\in I$, the prior on messages $z\in\isp(\tset_{idx})$ only depends on those training samples, $\tset_\idx\subset\tset$, indexed by $\idx$ and \textit{not} the whole training set. The same factorization is likewise required of $Q$. In fact, we will assume throughout that any distribution on $\domain$ has this factorization.

Given training set $\tset$ and posterior $Q=Q_IQ_{\isp(\tset)}$ (possibly depending on $\tset$) the Gibbs classifier $G_Q$ classifies new $x$ by sampling $\idx\sim Q_I(\idx)$, $z\sim Q_{\isp(\tset_\idx)}(z)$, setting $\hyp=R(\idx,z)$, and returning the label $\hyp(x)$.

In analogy with the data independent setting,
the goal is again to claim that the empirical Gibbs risk $\risk{\tset}{G_Q}$ is close to the true Gibbs risk $\risk{\pdata}{G_Q}$ when $KL(Q||P)$ is small compared to the number of samples $m$. This is the content of Theorem 3 in \cite{Laviolette2007}, which, though more general than we require, we cite verbatim for reference. For example, the theorem uses notation $\bar{Q}$ and $d_{\bar{Q}}$, which will simplify to $\bar{Q}=Q$ and $d_{\bar{Q}}=s$ in our more specialized setting where $P,Q$ have nonzero weight only for $|\idx|=s$. A specialized version to follow:

\begin{theorem}{(Theorem 3 in \cite{Laviolette2007})}\\
For any $\delta\in(0,1]$, for any reconstruction function mapping compression sequences and messages to classifiers, for any $\tset\in(\mathcal{X}\times\mathcal{Y})^m$ and for any prior $P_{\tset}$ on $I\times\isp_{\tset}$, we have 

\begin{align}
&\Prob{\forall Q\quad \Phi(Q,P,m,\delta)  }\geq 1-\delta\nonumber\\
\intertext{where $\Phi(Q,P,m,\delta)$ is the event}
&kl(\risk{\tset}{G_Q}||\risk{\pdata}{G_Q}) \leq  
\frac{KL(\bar{Q}||P) + \ln\frac{m+1}{\delta}}{m-d_{\bar{Q}}}\nonumber
\end{align}
\end{theorem}

In the special case we consider where $P,Q$ have nonzero weight only for $|\idx|=s$, we have the reduction $\bar{Q}=Q$ and $d_{\bar{Q}}=s$. More specialized still, we consider \cite{Laviolette2007} Theorem 9, which specializes to the case where $G_Q$ achieves zero training error. It is \textit{slightly} tighter than simply plugging in $0$ for $\risk{\tset}{G_Q}$ by an additive factor of $\ln(m+1)/m$. Notice in the following that the form of the bound arises because $kl(0||R)=-\ln(1-R)$:

\begin{theorem}{(Special case of Theorem 9 in \cite{Laviolette2007})}\label{thm:lav_simple}\\
Fix $s\leq m$. Let $I_s\subset I$ be the set of $s$-sized subsets of indices $[m]$. For any $\delta\in(0,1]$, for any reconstruction function mapping compression sequences and messages to classifiers, for any fixed prior $P_T$ that defines for every arbitrary sequence $T\in{\mathcal{X}\times\mathcal{Y}}^m$
a distribution on $I_s\times\isp_{\tset}$ , we have 

\begin{align}
&\Prob{\forall\{Q:\risk{\tset}{G_Q}=0\} \quad \Phi(Q,P,m,\delta,s)  }\nonumber\\
&\quad \quad \geq 1-\delta\nonumber\\
\intertext{where $\Phi(Q,P,m,\delta,s)$ is the event}
&\risk{\pdata}{G_Q}\leq 1-exp
\left[-\frac{KL(Q||P_{\trainS}) + \Ln{\frac{1}{\delta}}}{m-s} \right]
\label{eqn:lav_zte}
\end{align}
\end{theorem}

Notice though that
\begin{equation}
0\leq-\ln(1-\Risk)-\Risk\leq\epsilon(\Risk)
\end{equation}
, where $\epsilon(\Risk)\leq 0.03$ for the reasonable operating range, $\Risk\leq 0.2$. Therefore, as a matter of taste, in place of Equation \ref{eqn:lav_zte} in the above theorem we can claim the (very slightly) weaker but notationally more compact bound:  
\begin{align}
&\risk{\pdata}{G_Q} \leq \frac{KL(Q||P_{\trainS}) + \Ln{\frac{1}{\delta}}}{m-s}\label{eqn:lav_zte_compact}
\end{align}

In fact as long as  for in the range of $\Risk$, we would
Now we are ready to prove our main theorem.

\subsection{Theorem \ref{thm:pbsc}:A Neural Network Sample Compression Bound}\label{app:pbsc}

\def\idxhat{\textbf{i}_{\net}}
We restate and prove Theorem \ref{thm:pbsc} from Section \ref{sec:pbsc}.
\PBSC*

\begin{proof}
We start by defining without reference to a training set: our reconstruction function, our base space, and a fixed prior $P_T$ for every possible sequence $T\subset(\mathcal{X}\times\mathcal{Y})^m$.

Let $T\in(\mathcal{X}\times\mathcal{Y})^m$ be arbitrary.
Let $\Is$ be the set of subsets of $s$ elements from $[m]$. Let $\isp^\sigma(T)$ be the set of tuples of neuron states for inputs $T$ that are achievable with at least some network weights: $\isp^\sigma(T)\triangleq\{(\sigbar(x,v))_{(x,y)\in T}:v\in\WeightSpace\}$.

For future convenience, define a "max-margin conditional", $\cond(T_\idx,\Sigma)$, to be "True" iff there exist a nonempty set of weights $\WeightSpace(T_\idx,\Sigma)\subset\WeightSpace$ such that $\forall v\in\WeightSpace(T_\idx,\Sigma)$:
(1) $\Sigma=(\sigbar(x,v))_{(x,y)\in T}$
and (2) $\Lambda(v)$ is the max-margin classifier for $\{(\phi(x,v),y)\}_{(x,y)\in T}$. Put $\kappa(T_\idx,\Sigma)=|\{\net^{\sign}_v :v\in\WeightSpace(T_\idx,\Sigma)\}|$ to be the number of neural network classifiers obtained from some model parameter in $\WeightSpace(T_\idx,\Sigma)$. Note that $\kappa(T_\idx,\Sigma)\leq 2^n$ by Theorem \ref{thm:skeleton}.

For $\Sigma\in\isp^\sigma(T)$, if $\cond(T_\idx,\Sigma)$ is True, put $\isp^\skisp(T,\Sigma)=[\kappa(T_\idx,\Sigma)]$ and put $\isp^\skisp(T,\Sigma)=[1]$ otherwise.

Let our prior $P_T(\idx,\Sigma,\pth)$ have support on $\Is\times\isp^\sigma_T\times\isp^\skisp_T$, where the component spaces are defined as
\begin{align}
\isp^\sigma_T&\triangleq \bigcup_{\idx\in\Is} \isp^\sigma(T_\idx).\nonumber\\
\isp^\skisp_T&\triangleq \bigcup_{\idx\in\Is \Sigma\in\isp^\sigma(T_\idx)}
\isp^\skisp(T_\idx,\Sigma)\nonumber
\end{align}

where the prior $P_T$ has the factorization $P_T(\idx,\Sigma,\pth)=\Pi(\idx)\Ps_{T_\idx}(\Sigma)\Ppi_{(T_\idx,\Sigma)}(\pth)$, where $\Pi$ is not allowed to depend on the sequence $T$, and each factor distribution is uniform on the corresponding set of allowable messages:
\begin{align}
\Pi &= \Uniform(\Is)\nonumber\\
\Ps &= \Uniform(\isp^\sigma(T_\idx))\nonumber\\
\Ppi &= \Uniform(\isp^\skisp(T_\idx,\Sigma))\nonumber\\
\end{align}

Our reconstruction function maps each $(\idx,\Sigma,\pth)\in\Is\times\isp^\sigma_T\times\isp^\skisp_T$ to a classifier as follows: if $\cond(T_\idx,\Sigma)$ is True, $\Ron(T_\idx,\Sigma,\pth)$ returns the $\pthh$
classifier, in $\{\prednet :w\in\WeightSpace(T_\idx,\Sigma)\}$ (any total ordering on network classifiers $\{\net^{\sign}_v:v\in\WeightSpace\}$, can be used to clarify the meaning of $\pthh$). Else if $\cond(T_\idx,\Sigma)$ is False, $\Ron(T_\idx,\Sigma,\pth)$ returns a "dummy" classifier. To make a concrete choice, return the constant classifying function: $\Ron(T_\idx,\Sigma,\pth)=(x\mapsto+1)$ if $\cond(T_\idx,\Sigma)$ False.

Only now, let $\tset$ be a training set sampled from $\pdata^m$. Consider now only the "posterior" distributions $Q$ on $\Is\times\isp^\sigma_{\tset}\times\isp^\skisp_{\tset}$ that satisfy $\support Q\subset\support P_{\tset}$ and factorize according to $Q(\idx,\Sigma,\pth)=\Qi(\idx)\Qs_{\tset_\idx}(\Sigma)\Qpi_{(\tset_\idx,\Sigma)}(\pth)$. Note, in contrast with the prior, here each of $\Qi, \Qs, \Qpi$ are allowed to depend on the samples $\tset$. Let $G_Q$ be the Gibbs classifier which classifies $x$ stochastically by sampling $(\idx,\Sigma,\pth)\sim Q(\idx,\Sigma,\pth)$ and returning $\Ron(\idx,\Sigma,\pth)(x)$.

Then, from Theorem \ref{thm:lav_simple} and Equation \ref{eqn:lav_zte_compact}, we know that $\forall\delta\in(0,1]$,
\begin{align}\label{eqn:pbsc1}
&\Prob{\forall\{Q:\risk{\tset}{G_Q}=0\}\Phi(Q,P,m,\delta,s)}\geq 1-\delta\nonumber\\
\intertext{where $\Phi(Q,P,m,\delta,s)$ is the event}
&\risk{\pdata}{G_Q} \leq \frac{KL(Q||P_{\trainS}) + \Ln{\frac{1}{\delta}}}{m-s}
\end{align}

\def\Qnet{Q_{\net}}
Consider the weights $w$ classifier $\prednet$ we obtain from training the neural network $\net$ on $\tset$. Since the above bound is uniform over all $Q$, if we can find a posterior $\Qnet$ such that $G_{\Qnet}=\prednet$, then we can use Equation \ref{eqn:pbsc1} to bound the true risk $\risk{\pdata}{\prednet}$ of our neural network. Else if we cannot, then Equation \ref{eqn:pbsc1} does not comment on the risk $\risk{\pdata}{\prednet}$. However, we will show that whenever the three assumptions of the theorem hold, we can find such a posterior, and the bound will hold.

Well, by \Aone, we know that $\Lambda(w)$ is the unique max-margin classifier for $\trainembed$. But, since we have also assumed at most $s$ network support vectors, we know that $\Lambda(w)$ is \textit{also} the unique max-margin classifier for some subset of support vectors $\trainSsup\subset\tset$. Since $|\trainSsup|\leq s$, we can get $\idxhat\in\Is$ such that $\trainSsup\subset\tset_{\idxhat}$. Furthermore, there is at least one value $\Sigma_{\net}\in\isp^\sigma_{\tset}$, namely $\Sigma_{\net}\triangleq(\sigbar(x,w))_{(x,y)\in \tset_{\idxhat}}$, for which $\cond(\tset_{\idxhat},\Sigma_{\net})$ is True and $\WeightSpace(\tset_{\idxhat},\Sigma_{\net})\ni w$ is nonempty. Hence, for some $\pthhat\in\isp^\skisp_{\tset}$, $\Ron(\idxhat,\Sigma_{\net},\pthhat)=\prednet$ as desired.

\def\SymSig{Sym(w,\tset_{\idxhat})}
Let $\Qi$ be a distribution on $\Is$ which samples the index set $\idxhat$ with probability $1$. Let $\Qs$ be a distribution on $\isp^\sigma_{\tset}$ which is uniform over the set of activations consistent with $\prednet$:
\begin{align}\label{eqn:perm_neuron}
\SymSig&\triangleq\{\Sigma:\exists v\in\WeightSpace(\tset_{\idxhat},\Sigma)\\ & \quad \text{ with } \net^{\sign}_v=\prednet\}\nonumber\\
\Qs&=\Uniform(\SymSig).
\end{align}
For example, within-layer neuron permutations yield different $\Sigma$ but the same classifier. At last, for each $\Sigma\sim\Qs$, let $\Qpi_{(\tset_\idx,\Sigma)}$ be a distribution on $\isp^\skisp_{\tset}$ placing all of its mass on the (unique) index $\pth_{\Sigma}$ such that $\Ron(\idxhat,\Sigma,\pth_{\Sigma})=\prednet$ as functions. Therefore, $\Qnet\triangleq\Qi\Qs\Qpi$ is a posterior distribution returning $\prednet$ with probability one. Thus the Gibbs classifier $G_{\Qnet}$ is a deterministic classifier and is equal to $\prednet$.

There, we may claim Equation \ref{eqn:pbsc1} holds for posterior $\Qnet$ with probability at least $1-\delta$.

To conclude our theorem, we simply expand and upper bound $KL(\Qnet,P_{\tset})$:
\def\Ai{\Ave{\idx\sim\Qi}}
\def\As{\Ave{\Sigma\sim\Qs}}
\def\Ap{\Ave{\pth\sim\Qpi}}
\begin{align}
&KL(\Qnet||P_{\tset})=\nonumber\\
&=\Ai\As\Ap\ln\left(\frac{\Qi(\idx)\Qs(\Sigma)\Qpi(\pth)}{\Pi(\idx)\Ps(\Sigma)\Ppi(\pth)}\right)\nonumber\\
&=\Ai\ln\left(\frac{\Qi(\idx)}{\Pi(\idx)}\right) + \Ai\As\ln\left(\frac{\Qs(\Sigma)}{\Ps(\Sigma)}\right)\nonumber \\
&\quad+\Ai\As\Ap\ln\left(\frac{\Qpi(\pth)}{\Ppi(\pth)}\right)\nonumber\\
&=\Ln{\choosemany} + \Ai\As\ln\left(\frac{|\isp^\sigma(\tset_{\idxhat})|}{|\SymSig|}\right)\nonumber\\
&+\Ai\As\Ap\ln\left(|\isp^\skisp(\tset_{\idxhat},\Sigma)|\right).\label{eqn:exactKL}
\end{align}

To conclude the proof, we crudely upper bound $\forall\idx$ $|\isp^\sigma(\tset_{\idx})|\leq 2^{ns}$, which follows because at each of $s$ samples $(x,y)\in\tset_{\idxhat}$ and at each neuron, $(l,i_l)$, of $n$ possible neurons, $\sigl(x,w)_{i_l}$ can take one of two values. We also have $|\isp^\sigma(T_\idx)|\leq 2^n$ by Theorem \ref{thm:skeleton}. Clearly, $|\SymSig|\geq 1$. Combining this with Equation \ref{eqn:exactKL}, we have 
\begin{align}
KL(\Qnet||P_{\tset})&\leq \Ln{\choosemany}+\ln(2^{ns}2^n)\nonumber\\
\intertext{where we simply drop $\ln(2)<1$, and approximate $\choosemany\leq (\frac{me}{s})^s$ to arrive at}
KL(\Qnet||P_{\tset})&\leq s\Ln{\frac{m}{s}}+s+ns+n \nonumber
\end{align}
Substituting the above into Equation \ref{eqn:pbsc1} finishes the proof.
\end{proof}

\end{document}